\documentclass{article}

\usepackage{PRIMEarxiv}

\usepackage[utf8]{inputenc} 
\usepackage[T1]{fontenc}    
\usepackage{url}            
\usepackage{booktabs}       
\usepackage{amsfonts}       
\usepackage{nicefrac}       
\usepackage{microtype}      
\usepackage{lipsum}
\usepackage{fancyhdr}       
\usepackage{graphicx}       
\graphicspath{{media/}}     

\usepackage{amsmath}
\usepackage{subcaption}
\usepackage[authoryear]{natbib}
\usepackage[colorlinks]{hyperref}
\usepackage[ruled,vlined,linesnumbered]{algorithm2e}
\usepackage{booktabs, makecell, multirow, tabularx}
\usepackage{appendix}
\usepackage{xcolor}
\usepackage{makecell}
\usepackage{amsthm}
\newtheorem{theorem}{Theorem}

\title{Knowledge Transfer from Simple to Complex: A Safe and Efficient Reinforcement Learning Framework for Autonomous Driving Decision-Making
}

\author{
  Rongliang Zhou \textsuperscript{1}, Jiakun Huang \textsuperscript{1}, Mingjun Li \textsuperscript{1}, Hepeng Li \textsuperscript{1}, Haotian Cao \textsuperscript{2} \thanks{\textbf{Corresponding author}} , Xiaolin Song \textsuperscript{1}  \footnotemark[1]\\
  \textsuperscript{1} State Key Laboratory of Advanced Design and Manufacturing Technology for Vehicle, Hunan University\\
  \textsuperscript{2} College of Intelligence Science and Technology, National University of Defense Technology\\
}

\begin{document}
\maketitle

\makeatletter
\renewcommand{\@makefnmark}{}
\makeatother

\footnotetext{\texttt{Email addresses: zhourongliang@hnu.edu.cn (Rongliang Zhou), hjk0517@hnu.edu.cn (Jiakun Huang), mingjunl@hnu.edu.cn (Mingjun Li), lhphnu@hnu.edu.cn (Hepeng Li), caohaotian@nudt.edu.cn (Haotian Cao), jqysxl@hnu.edu.cn (Xiaolin Song)}}

\begin{abstract}
\label{abstract}
A safe and efficient decision-making system is crucial for autonomous vehicles. However, the complexity of driving environments limits the effectiveness of many rule-based and machine learning approaches. Reinforcement Learning (RL), with its robust self-learning capabilities and environmental adaptability, offers a promising solution to these challenges. Nevertheless, safety and efficiency concerns during training hinder its widespread application. To address these concerns, we propose a novel RL framework, Simple to Complex Collaborative Decision (S2CD). First, we rapidly train the teacher model in a lightweight simulation environment. In the more complex and realistic environment, teacher intervenes when the student agent exhibits suboptimal behavior by assessing actions' value to avert dangers. We also introduce an RL algorithm called Adaptive Clipping Proximal Policy Optimization Plus, which combines samples from both teacher and student policies and employs dynamic clipping strategies based on sample importance. This approach improves sample efficiency while effectively alleviating data imbalance. Additionally, we employ the Kullback-Leibler divergence as a policy constraint, transforming it into an unconstrained problem with the Lagrangian method to accelerate the student's learning. Finally, a gradual weaning strategy ensures that the student learns to explore independently over time, overcoming the teacher's limitations and maximizing performance. Simulation experiments in highway lane-change scenarios show that the S2CD framework enhances learning efficiency, reduces training costs, and significantly improves safety compared to state-of-the-art algorithms. This framework also ensures effective knowledge transfer between teacher and student models, even with suboptimal teachers, the student achieves superior performance, demonstrating the robustness and effectiveness of S2CD.
\end{abstract}





\section{Introduction}
\label{sec:introduction}
Given the unpredictable and complex nature of driving environments, nearly 94\% of traffic incidents are related to poor decision-making by human drivers \citep{alvaro2018driver}.
Suboptimal driving decisions negatively impact both the safety and efficiency of traffic flow \citep{lee2004comprehensive}.
With the rapid progress in autonomous driving technology, considerable interest has been drawn to its promise for enhancing both traffic safety and efficiency. 
Autonomous vehicle systems generally consist of a perception layer, a decision-making layer, and a planning and control layer, in which the decision-making layer commonly referred to as the "brain" of the system \citep{chen2017brain}.
Therefore, the performance of decision-making algorithms directly relates to the effectiveness of intelligent vehicle technologies \citep{li2017game}.
Currently, decision-making methods for autonomous vehicles are generally classified into two categories: knowledge-driven and data-driven.
Knowledge-driven methods, such as Hierarchical State Machines (HSM) \citep{patz2008practical}, Expert Systems (ES) \citep{ferguson2008reasoning}, and Finite State Machines (FSM) \citep{buehler2009darpa}, are known for their rigorous logic and high interpretability \citep{bae2020finite}.
However, these methods are limited by their strong reliance on existing knowledge, which restricts their ability to handle novel or unexpected scenarios \citep{bianco2005knowledge, mozina2008fighting}.
In contrast, data-driven methods such as End-to-End (E2E) Learning \citep{bojarski2016end}, Imitation Learning (IL) \citep{hawke2020urban}, and Reinforcement Learning (RL) \citep{sallab2017deep, zhang2022receding} have gained prominence due to their robust self-learning capabilities and adaptability to dynamic environments.
RL algorithms, in particular, optimize driving strategies through real-time interactions with the environment, making them highly effective in complex scenarios \citep{zhu2021survey}.
While RL has demonstrated success in fields such as game artificial intelligence (AI) \citep{zhang2022leader, ye2020mastering}, financial transactions \citep{tsantekidis2023modeling, pendharkar2018trading}, and energy management \citep{yu2019deep, kuznetsova2013reinforcement}, its application in safety-critical areas like autonomous driving and robotic control faces significant challenges due to safety risks and training inefficiency \citep{liu2022constrained, ou2024modular}.

RL, much like human behavior, can autonomously explore environments and optimize strategies based on environmental feedback \citep{subramanian2022reinforcement}.
However, traditional RL algorithms inevitably face risks during exploration, often encountering dangers before learning how to avoid them \citep{chow2018lyapunov}.
The cost of allowing a vehicle to undertake dangerous actions in real-world scenarios is prohibitively high, making it challenging to apply traditional RL algorithms in autonomous driving.
In human learning, particularly in hazardous situations, we do not rely solely on trial and error to acquire knowledge \citep{peng2022safe}.
Instead, we seek guidance from a teacher to ensure safety and efficiency during the learning process.
The teacher not only demonstrates correct actions when necessary but also directly corrects the student's mistakes to avert dangerous situations.
For instance, a driving instructor can take control of a vehicle when a novice driver makes an unsafe decision, thereby preventing accidents.
Consequently, numerous researchers have incorporated the Teacher-Student Framework (TSF) \citep{zimmer2014teacher, liu2023traffic} into RL, enabling student agents to safely and efficiently navigate complex scenarios through autonomous exploration and teacher guidance.
In the TSF, a high-performance neural network or human expert acts as the teacher, intervening when the student's actions meet specific criteria.
When an intervention occurs, the teacher's actions replace the student's actions in the replay buffer as training data, which is later used to optimize the student's model via RL algorithms.
However, human-in-the-loop TSF requires continuous expert monitoring, which is labor-intensive, and the effectiveness of human teachers can be influenced by subjective factors like emotions and mental state.
To overcome these issues, most TSF approaches rely on pre-training a high-performance teacher model using neural networks, allowing the student to learn correct behaviors from the teacher’s demonstrations.
Yet, in fields like robotic control and autonomous driving, obtaining such a high-performance teacher in advance can be extremely expensive.
Furthermore, in current TSF methods, if the teacher provides suboptimal demonstrations, the student may learn incorrect strategies, limiting its performance \citep{xue2023guarded}.
Another limitation of current TSF methods is that the RL algorithm can only utilize the Markov Decision Process (MDP) corresponding to the action chosen by either the teacher or the student at each step, leading to significant data wastage since alternative actions are not leveraged.

In summary, traditional RL algorithms involve high training costs and fail to ensure the safety of agents.
While standard TSF methods help student agents learn safely and efficiently, they come with significant costs in developing a high-performance teacher, and the student's ultimate performance is constrained by the teacher's abilities. 
Furthermore, these methods cannot optimize the policy by simultaneously utilizing samples generated by both the teacher and student models.
To tackle these shortcomings, this paper introduces an innovative TSF framework called Simple to Complex Collaborative Decision (S2CD), which enhances knowledge transfer between teacher and student.
Unlike other TSFs, the S2CD framework accelerates training by first training the teacher in a simple simulation environment with similar scenarios, then guiding the student in more complex simulations or real-world settings.
We also theoretically demonstrated that teacher guidance improves both the performance and safety ceilings of the student during training.
Additionally, we propose the Adaptive Clipping Proximal Policy Optimization Plus (ACPPO+), which builds upon our previous work, ACPPO \citep{zhou2024hybrid}, by leveraging data from both teacher and student models to further improve sample efficiency and mitigate data imbalance.
The algorithm adaptively adjusts the clipping factor based on sample importance, boosting learning efficiency.
During the policy optimization process, the neural network is updated by constraining the Kullback-Leibler (KL) divergence between the teacher’s and student’s policies, facilitating the rapid convergence of the student policy to the teacher’s.
Finally, S2CD employs a weaning strategy that allows the student agent to independently explore the environment in later training stages, thus overcoming the teacher’s limitations.
The main contributions of this paper are summarized as follows:
\begin{itemize}
\item
This paper presents a novel TSF called S2CD, which 
accelerates the training of a teacher model in a lightweight simulation environment and uses it to guide the student agent in learning safely and efficiently in more complex simulated or real-world environments.
\item
This paper proposes the ACPPO+ algorithm, integrated within the S2CD framework, which simultaneously utilizes data from both teacher and student models while adaptively adjusting the clipping factor according to sample importance, thus enhancing sample efficiency and mitigating data imbalance.
\item
This paper introduces the use of the KL divergence between the teacher's and student's policies as a constraint during policy updates, and transforms it into an unconstrained problem using the Lagrangian method to facilitate the rapid convergence of the student’s policy to the teacher’s.
\item
This paper relies on action values to trigger teacher policy intervention and gradually reduces the probability of intervention through a weaning mechanism, encouraging independent exploration by the student and mitigating the limitations imposed by suboptimal teacher performance.
\end{itemize}

\autoref{sec:related_works} reviews related work and analyzes the strengths and weaknesses of these studies to underscore the value of our research.
\autoref{sec:theoretical_background} provides a detailed explanation of the theoretical background of the algorithms, followed by an introduction to the S2CD decision-making framework.
\autoref{sec:implementation} describes the differences between the Highway-Env and Carla simulation environments and provides details on the implementation of the experiments.
\autoref{sec:simulation_result} shows the results of the simulation experiments and compares them with state-of-the-art (SOTA) algorithms.
\autoref{sec:conclusion} concludes the paper, explores the limitations of the study, and suggests directions for future research.

\section{Related Works}
\label{sec:related_works}
The central idea of this work is to facilitate the safe and efficient learning of RL agents through a knowledge transfer technique based on the TSF.
Therefore, this section reviews studies related to knowledge transfer and Safe RL, analyzing their strengths and weaknesses to highlight the innovative aspects and practical significance of this research.

\subsection{Knowledge Transfer}
While RL has been widely applied across various fields and has demonstrated impressive results, most research has primarily focused on simulation experiments.
However, in the context of autonomous driving, the ultimate goal is real-world application, where safety and efficiency during training process are critical factors in assessing algorithm performance.
Training standard RL algorithms on real vehicles would not only risk substantial damage but also incur significant time costs.
Knowledge transfer seeks to improve the learning efficiency of models by transferring knowledge from previously learned tasks to new tasks.
By leveraging existing policies or experiences, the new task can accelerate the learning process, thereby reducing the need for extensive data collection or large sample sizes.
This technique is especially valuable in environments where acquiring training data is costly or time-consuming.
Consequently, effectively transferring knowledge from simulation-based training to agents operating in real-world or more complex simulation environments has become a significant focus of research.

For example, \citet{qiao2018automatically} and \citet{song2021autonomous} employed Curriculum Learning (CL) strategies to achieve notable performance in urban intersection and overtaking tasks, respectively.
CL enhances model learning effectiveness by gradually increasing task difficulty, instead of requiring models to tackle complex tasks directly \citep{bengio2009curriculum}.
However, CL requires meticulously designed multi-level tasks.
Furthermore, the progression from simpler to more complex tasks can increase training time and computational resource demands.
Transfer Learning (TL) is another knowledge transfer technique that effectively reduces the training costs of the target task by applying knowledge acquired from the source task to the target task.
For instance, \citet{shu2021driving} combined RL with TL to transfer decision-making knowledge from one driving task to another in an intersection environment.
Similarly, \citet{akhauri2020enhanced} utilized simulated accident data for TL, enabling the model to generalize more rapidly and effectively to real-world scenarios.
However, if the target and source tasks are not adequately related, the performance of TL may deteriorate, potentially resulting in "negative transfer" \citep{zhang2022survey}.
Another approach to knowledge transfer is the use of Digital Twins (DT).
By creating an accurate digital replica of a physical entity or system, we can map data from the physical world to the virtual world \citep{niaz2021autonomous}, enabling real-time monitoring, simulation, analysis, and optimization of physical entities.
\citet{voogd2023reinforcement} employed DT technology to combine virtual and real-world data on vehicle dynamics and traffic scenarios, effectively bridging the gap between simulation and reality. 
\citet{wang2024digital} designed and implemented an E2E DT system for autonomous driving.
This system captures real-world traffic information, processes it in the cloud to create a digital twin model, and then utilizes this model to provide route planning services for autonomous vehicles.
Despite DT technology effectively bridging the gap between simulation training and real-world applications, accurate simulation processes remain indispensable.
Designing simulation programs that accurately mirror or fully replicate real-world conditions is both challenging and costly \citep{yun2021simulation}.
Moreover, pre-training a high-performance model in a high-fidelity simulation environment is resource-intensive.

Therefore, a key area of research focuses on developing methods to train models quickly and efficiently in simplified simulation environments.
Subsequently, effective knowledge transfer techniques are applied to enable agents to complete training safely and effectively in more complex simulations or real-world environments.

\subsection{Safe RL}
RL methods can enhance an agent's ability to avoid dangers by learning from previous training failures and developing efficient behavioral strategies to adapt to complex environments.
However, the application of these methods in safety-critical domains continues to present substantial safety challenges \citep{garcia2015comprehensive}.
Ensuring policy safety throughout the training process and preventing agents from encountering dangerous situations are key to expanding the applicability of RL to more fields.
Safe RL seeks to ensure that RL models avoid unsafe behaviors during both training and execution, particularly in high-risk or safety-critical domains such as autonomous driving.
Safe RL techniques focus on balancing the exploration of novel strategies with the imperative to maintain safety, ensuring that the learned policies adhere to required safety standards throughout the learning process.
Consequently, many researchers are focusing on incorporating constrained optimization, hard rule constraints, and risk assessment into RL algorithms to mitigate unsafe behaviors during both the learning and execution phases.

For instance, methods like SAC-Lag \citep{ha2021learning} and PPO-Lag \citep{stooke2020responsive} utilize the Lagrangian method, which simultaneously updates both the policy and the Lagrange multiplier through primal-dual optimization, thereby transforming the safety-constrained optimization problem into an unconstrained one.
\citet{peng2022model} proposed the Separated Proportional Integral Lagrangian (SPIL) algorithm, which mitigates oscillations and reduces conservatism in RL policies during car-following scenarios.
\citet{zhang2024safe} introduced cognitive uncertainty as a constraint in Lagrangian Safe RL algorithms, aiming to enhance both exploration and safety performance in autonomous vehicles.
While the Lagrangian method is effective in handling safety constraints, it lacks a clear mechanism to prevent hazardous events \citep{li2022efficient}.
Thus, agents may still engage in unsafe exploration during training.
Consequently, many researchers have applied rule-based constraints \citep{cao2022trustworthy,likmeta2020combining} or risk estimation methods \citep{mo2021safe, li2022decision} to Safe RL in autonomous driving, aiming to ensure a minimum level of policy safety performance.
Although these methods can prevent a significant number of unsafe actions, their effectiveness heavily depends on the quality of the rule design or risk assessment models.
To further reduce unsafe events, researchers have shifted their focus toward Safe RL approaches based on the TSF.
Within the TSF, teacher policies can be classified into human expert strategies and well-trained neural network strategies.
For example, \citet{wu2023toward} and \citet{huang2024human} employed human experts to mentor AI agents.
During the student agent's interaction with the environment, the human expert intervenes in dangerous situations and demonstrates the correct actions, thereby ensuring driving safety throughout the student's training process.
While Safe RL methods involving human expert intervention significantly improve the safety of the exploration process, they incur high labor costs.
Moreover, the teaching performance of human experts is highly influenced by subjective factors, leading to instability in the final model performance.
To reduce labor costs, many researchers have employed well-trained neural network models as "teachers" to guide student agents during the learning process.
For instance, \citet{peng2022safe} and \citet{xue2023guarded} employed RL algorithms to obtain a well-trained teacher model, when the conditions defined by a switching function, which determines whether to intervene, are met, the teacher takes over the student's actions to demonstrate correct operations and avoid dangerous situations.
\citet{zhou2024hybrid} employed the SOAR cognitive architecture as a teacher to guide the student agent toward efficient and safe learning, while also offering a certain degree of interpretability.

Model-based TSF eliminates the need for real-time monitoring by human experts but incurs the cost of training a high-performance model.
Moreover, the student's final performance is limited by the teacher's capabilities.
Therefore, although TSF can maximize agent safety during training and improve learning efficiency, these methods have not yet been widely adopted in the field of autonomous driving.

\section{Theoretical Background}
\label{sec:theoretical_background}
In this section, we offer a comprehensive explanation of the technical background underlying the S2CD framework.
First, we present the principle of the PPO algorithm.
Subsequently, we introduce the ACPPO+ algorithm.
Lastly, we offer a detailed description of the components within the S2CD framework.

\subsection{Proximal Policy Optimization Algorithm}
DRL methods are generally categorized into two types: value-based and policy-based methods.
The core principle of value-based methods is to learn a value function that estimates the long-term return for each state or state-action pair, which is then used to guide action selection.
In contrast, policy-based methods directly optimize a policy that maps states to actions, enabling decision-making based on this learned policy.
Although value-based RL methods perform well in many simple tasks, their efficiency and accuracy significantly decrease as the state or action space becomes more complex.
Conversely, policy-based methods provide a clear optimization pathway, ensuring convergence and enhancing exploration through stochastic policies.
For complex tasks, policy-based methods are generally more suitable and offer greater scalability.
Therefore, this study adopts policy-based RL methods.

The objective of RL is to identify a policy $\pi(a|s)$ that maximizes the cumulative discounted reward across all time steps, defined as:
\begin{equation}
	\centering
	G_{t} = \sum_{k=0}^{\infty}\gamma^{k}r_{t+k}
\end{equation}
\noindent
where $r_{t}$ represents the reward received at time step $t$, and $\gamma\in[0,1)$ is the discount factor, which determines the importance of future rewards.
In policy gradient methods, the objective is to optimize the policy $\pi(a|s)$ to maximize the expected cumulative reward, expressed as:
\begin{equation}
	\label{eq:J1}
	\centering
	J(\theta) = \mathbb{E}_{s_0} \left[V^{\pi_\theta}(s_{0}) \right] = \mathbb{E}_{\pi_\theta}\left[\sum_{t=0}^{\infty}\gamma^{t}r(s_{t}, a_{t})\right]
\end{equation}
\noindent
where $\mathbb{E}$ is the expectation for a batch of samples, $\pi_\theta$ denotes the current policy, with $\theta$ representing its parameters.
The state value function $V^\pi(s_t) = \mathbb{E}_\pi \left[ \sum_{k=0}^{\infty} \gamma^k r(s_{t+k}, a_{t+k}) \middle| s_t \right]$ represents the expected return when starting from state $s_t$ under policy $\pi_\theta$.

The optimization objective under the policy $\pi_\theta$, denoted as 
$J(\theta)$, can be expressed in the expected form of the new policy $\pi_{\theta'}$:
\begin{equation}
	\label{eq:J2}
	\begin{split}
		J(\theta) = \mathbb{E}_{s_0} \left[ V^{\pi_\theta}(s_0) \right]
		= \mathbb{E}_{\pi_{\theta'}} \left[ \sum_{t=0}^{\infty} \gamma^t V^{\pi_\theta}(s_t) - \sum_{t=1}^{\infty} \gamma^t V^{\pi_\theta}(s_t) \right] 
		= -\mathbb{E}_{\pi_{\theta'}} \left[ \sum_{t=0}^{\infty} \gamma^t \left( \gamma V^{\pi_\theta}(s_{t+1}) - V^{\pi_\theta}(s_t) \right) \right]
	\end{split}
\end{equation}

Based on the \autoref{eq:J1}--\autoref{eq:J2} and Paper \citet{kakade2002approximately}, we have:
\begin{equation}
	\begin{split}
		J(\theta') - J(\theta) &= \mathbb{E}_{s_0} \left[ V^{\pi_{\theta'}}(s_0) \right] - \mathbb{E}_{s_0} \left[ V^{\pi_\theta}(s_0) \right] \\
		&= \mathbb{E}_{\pi_{\theta'}} \left[ \sum_{t=0}^{\infty} \gamma^t r(s_t, a_t) \right] + \mathbb{E}_{\pi_{\theta'}} \left[ \sum_{t=0}^{\infty} \gamma^t \left( \gamma V^{\pi_\theta}(s_{t+1}) - V^{\pi_\theta}(s_t) \right) \right] \\
		&= \mathbb{E}_{\pi_{\theta'}} \left[ \sum_{t=0}^{\infty} \gamma^t \left( r(s_t, a_t) + \gamma V^{\pi_\theta}(s_{t+1}) - V^{\pi_\theta}(s_t) \right) \right]
	\end{split}
\end{equation}

Meanwhile, the advantage function is denoted as:
\begin{equation}
	\begin{split}
		\centering
		A^{\pi_\theta}(s_{t}, a_{t}) = Q^{\pi_\theta}(s_{t}, a_{t})-V^{\pi_\theta}(s_{t})
		= r(s_t, a_t) + \gamma V^{\pi_\theta}(s_{t+1}) - V^{\pi_\theta}(s_t)
	\end{split}
\end{equation}
\noindent
where $Q^\pi(s_t, a_t) = r(s_t, a_t) + \gamma \mathbb{E}_{s_{t+1}} \left[ V^\pi(s_{t+1}) \right]$ is the state-action value function, representing the expected return after taking action $a_t$ in state $s_t$, so:
\begin{equation}
	\begin{split}
		J(\theta') - J(\theta) = \mathbb{E}_{\pi_{\theta'}} \left[ \sum_{t=0}^{\infty} \gamma^t A^{\pi_\theta}(s_{t}, a_{t}) \right]
		= \sum_{t=0}^{\infty} \gamma^t \mathbb{E}_{s_t \sim P_t^{\pi_{\theta'}}} \mathbb{E}_{a_t \sim \pi_{\theta'}(\cdot | s_t)} \left[ A^{\pi_\theta}(s_t, a_t) \right]
	\end{split}
\end{equation}
\noindent
where $P^{\pi}_{t}(s)$ is the density function of state at time $t$.
Let $\nu^{\pi}(s)=(1-\gamma)\sum_{t=0}^{\infty}\gamma^{t}P^{\pi}_{t}(s)$ represent the discounted visitation frequencies of states, so:
\begin{equation}
	\begin{split}
		J(\theta') - J(\theta) = \dfrac{1}{1-\gamma}\sum_{s}\nu^{\pi_{\theta'}}(s)\sum_{a}\pi_{\theta'}(a|s)A^{\pi_{\theta}}(s, a)
		= \dfrac{1}{1-\gamma}\mathbb{E}_{s\sim\nu^{\pi_{\theta'}}}\mathbb{E}_{a \sim \pi_{\theta'}(\cdot | s_t)}[A^{\pi_{\theta}}(s, a)]
		\label{eq:new_old_gap}
	\end{split}
\end{equation}

Therefore, as long as we can find a new policy such that $\mathbb{E}_{s\sim\nu^{\pi_{\theta'}}}\mathbb{E}_{a \sim \pi_{\theta'}(\cdot | s_t)}[A^{\pi_{\theta}}(s, a)] \geq 0$, we can ensure that the policy performance improves monotonically, i.e., $J(\theta') \geq J(\theta)$.

\autoref{eq:new_old_gap} includes $\nu^{\pi_{\theta'}}(s)$, where $\pi_{\theta'}$ is the target policy we aim to solve for.
To facilitate the calculation, we substitute $\nu^{\pi_{\theta}}(s)$ for $\nu^{\pi_{\theta'}}(s)$ in $J(\theta)$.
Given the close similarity between the new and old policies, this approximation is considered acceptable.
Thus, we obtain the locally approximated objective function:
\begin{equation}
	\label{eq:new_obj}
	\begin{split}
		\centering
		L_{\theta}(\theta') = J(\theta) + \dfrac{1}{1-\gamma}\sum_{s}\nu^{\pi_{\theta}}(s)\sum_{a}\pi_{\theta'}(a|s)A^{\pi_{\theta}}(s, a)
		 = J(\theta) + \dfrac{1}{1-\gamma}\mathbb{E}_{s\sim\nu^{\pi_{\theta}}}\mathbb{E}_{a \sim \pi_{\theta'}(\cdot | s_t)}[A^{\pi_{\theta}}(s, a)]
	\end{split}
\end{equation}

Since $J(\theta) = \mathbb{E}_{\pi_{\theta}} \left[\sum_{t=0}^{\infty} \gamma^t r(s_{t}, a_{t})\right]$, the following objective function can be derived using importance sampling:
\begin{equation}
	\label{eq:final_obj}
	\begin{split}
		\centering
		L_{\theta}(\theta') &= \mathbb{E}_{s\sim\nu^{\pi_{\theta}}}\mathbb{E}_{a \sim \pi_{\theta'}(\cdot | s_t)}[A^{\pi_{\theta}}(s, a)] 
		=\mathbb{E}_{s\sim\nu^{\pi_{\theta}}}\mathbb{E}_{a \sim \pi_{\theta}(\cdot | s_t)} \left[\dfrac{\pi_{\theta'}(a|s)}{\pi_{\theta}(a|s)}A^{\pi_{\theta}}(s, a) \right]
	\end{split}
\end{equation}

Hence, by performing gradient descent on the objective function, the policy $\pi_{\theta}$ can be optimized.
However, the deep structure of neural networks can lead to abrupt policy shifts if an inappropriate step size is used when optimizing the objective function along the policy gradient, potentially causing significant performance degradation.
\citet{schulman2015trust} proposed the Trust Region Policy Optimization (TRPO) method, which optimizes $L_{\theta}(\theta')$ by imposing a constraint on the KL divergence between the new and old policies.
To simplify the computation, the average KL divergence can be used to relax the constraint on the maximum value of the KL divergence, so the optimization objective for TRPO is:
\begin{equation}
	\begin{aligned}
		&\max_{\theta'} L_\theta(\theta') \\
		&\text{s.t.} \mathbb{E}_{s \sim \nu^{\pi_{\theta}}} \left[ D_{KL} \left( \pi_{\theta}(\cdot | s), \pi_{\theta'}(\cdot | s) \right) \right] \leq \delta
	\end{aligned}
\end{equation}

Although TRPO effectively limits the size of policy updates, it is computationally complex and time-consuming.
To address this, the PPO algorithm \citep{schulman2017proximal}
constrains policy updates by maximizing a clipped objective function:
\begin{equation}
	\label{eq:ppo_obj}
	L^{CLIP}(\theta) = \mathbb{E}_{\pi_{\theta}} \left[ \min \left( r_t(\theta) A^{\pi_{\theta}}(s, a), \text{clip}\left(r_t(\theta), 1 - \epsilon, 1 + \epsilon\right) A^{\pi_{\theta}}(s, a) \right) \right]
\end{equation}
\noindent
where  $r_{t}(\theta)=\dfrac{\pi_{\theta'}(a_{t}|s_{t})}{\pi_{\theta}(a_{t}|s_{t})}$ is the probability ratio, which is clipped within the range $(1-\epsilon, 1+\epsilon)$ by the operator $clip$; $\epsilon$ is the clipping factor that controls the clipping range.

\subsection{Adaptive Clipping Proximal Policy Optimization Plus Algorithm}
\label{sec:ACPPO+}
According to \autoref{eq:ppo_obj}, the value of $L^{CLIP}(\theta)$ is primarily influenced by the $A^{\pi_{\theta}}$.
When $A^{\pi_{\theta}}$ is positive, it indicates that the action $a_{t}$ is better than the average level, making it a favorable choice.
In this case, maximizing $L^{CLIP}(\theta)$ along the gradient direction increases $r_{t}(\theta)$, thereby raising the probability of selecting $a_{t}$.
Conversely, when $A^{\pi_{\theta}}$ is negative, it implies that $a_{t}$ is below average and thus not an ideal choice.
In this situation, maximizing $L^{CLIP}(\theta)$ along the gradient direction decreases $r_{t}(\theta)$, reducing the probability of selecting  $a_{t}$.
PPO restricts $r_{t}(\theta)$ within the range $(1-\epsilon, 1+\epsilon)$ through clipping.

This clipping mechanism limits the magnitude of policy updates in a simple and effective manner, thereby avoiding performance degradation due to abrupt policy changes.
However, previous PPO-based algorithms typically set the clipping factor $\epsilon$ as a fixed value.
When samples come from different sources with significant importance differences, this fixed value cannot adapt flexibly to the characteristics of each sample, making it challenging to accurately control the clipping range.
To address this issue, we proposed the improved ACPPO algorithm in previous work \citep{zhou2024hybrid}.
This algorithm introduces an adaptive clipping mechanism that dynamically adjusts $\epsilon$ based on sample importance, using the probability distribution of different actions.
The modified objective function is as follows:
\begin{equation}
	\label{eq:acppo_obj}
	L^{ACLIP}(\theta) =\left\{
	\begin{array}{lr}
		\mathbb{E}_{\pi_{\theta}} \left[ \min \left( r_t(\theta) A^{\pi_{\theta}}(s, a), \text{clip}\left(r_t(\theta), 1 - (\epsilon + \epsilon'), 1 + (\epsilon - \epsilon') \right) A^{\pi_{\theta}}(s, a) \right) \right], & {if}\ a_t=a^{P}_{t} \\
		\mathbb{E}_{\pi_{\theta}} \left[ \min \left( r_t(\theta) A^{\pi_{\theta}}(s, a), \text{clip}\left(r_t(\theta), 1 - (\epsilon - \epsilon'), 1 + (\epsilon + \epsilon') \right) A^{\pi_{\theta}}(s, a) \right) \right], & {if}\ a_t=a^{E}_{t} \\
	\end{array}
	\right.
\end{equation}
\noindent
where action $a^{E}_{t}$ originates from the expert policy and is considered safe and reliable in most situations; action $a^{P}_{t}$ is generated by the PPO policy and may lead to suboptimal outcomes, especially in the early training stages when the policy network is not yet fully optimized; $\epsilon'$ is the adaptive clipping factor, defined as follows:
\begin{equation}
	\epsilon' = \psi\dfrac{\left(\pi_\theta(a^s_{t} | s_t) - \pi_\theta(a^t_{t} | s_t)\right) + 1}{2}, \quad 0 \leq \psi \leq \epsilon
\end{equation}
where $\pi_\theta(\cdot | s_t)$ is the probability distribution of different actions in state $s_t$; $\psi$ serves as an adaptive factor to regulate the clipping range.

Although ACPPO performs adaptive clipping based on sample importance, the samples collected at each step are derived solely from either $a^{P}_{t}$ or $a^{E}_{t}$.
When the expert policy fully intervenes, it is highly likely that all collected samples are from expert actions, which prevents the agent from learning to avoid incorrect decisions, leading to an imbalance between positive and negative samples.
In this study, we collect the state, actions $a^{t}_{t}$ and $a^{s}_{t}$ generated by both the teacher and student, and the corresponding return value at each step as MDP training samples.
This approach not only doubles the amount of training data compared to ACPPO, improving sample efficiency, but also collects data for different actions under the same state, effectively mitigating the positive-negative sample imbalance issue in ACPPO.
We refer to the improved algorithm as Adaptive Clipping Proximal Policy Optimization Plus (ACPPO+), with its objective function modified as follows:
\begin{equation}
	\label{eq:acppo+_obj}
	\begin{split}
		L^{ACLIP+}(\theta) &= \mathbb{E}_{\pi_\theta} \Bigg[ \mathbb{I}_{a_t^s} \cdot \min \left( r_t(\theta) A^{\pi_\theta}(s, a_t^s), \text{clip}\left(r_t(\theta), 1 - (\epsilon + \epsilon'), 1 + (\epsilon - \epsilon')\right) A^{\pi_\theta}(s, a_t^s) \right) \\ 
		& \quad + \mathbb{I}_{a_t^t} \cdot \min \left( r_t(\theta) A^{\pi_\theta}(s, a_t^t), \text{clip}\left(r_t(\theta), 1 - (\epsilon - \epsilon'), 1 + (\epsilon + \epsilon')\right) A^{\pi_\theta}(s, a_t^t) \right) \Bigg]
	\end{split}
\end{equation}
\noindent
where $\mathbb{I}_{a_t^s}$ and $\mathbb{I}_{a_t^t}$ are indicator variables representing whether actions $a_t^s$ and $a_t^t$ exist, respectively.
If the action exists, the value is 1; otherwise, it is 0.

As seen from \autoref{eq:acppo+_obj}, ACPPO+ allows for greater updates to policy parameters when training with more important teacher samples, while limiting the influence of student samples.
Through this adaptive clipping mechanism, ACPPO+ not only leverages data generated by both teacher and student policies but also effectively adapts to samples from various sources.
This accelerates the alignment of the student policy with the teacher policy, thereby improving the agent's learning efficiency.

\subsection{Simple to Complex Collaborative Decision-Making Framework}
In this section, we provide a detailed overview of the components of the S2CD framework, which include the offline teacher training module, the high-level layer, and the low-level layer.
In the teacher training module, an RL agent is rapidly trained in a lightweight simulation environment as the teacher model, with two neural networks used to approximate the Return function and the Q-Value function, respectively, to predict the return values and Q-values for state-action pairs.
The decision-making layer outputs final high-level commands based on the teacher’s action guidance and state information, while the low-level layer receives these commands to perform path planning and ultimately control the vehicle’s motion.
\autoref{fig:S2CD_framework} depicts the components of the S2CD framework.
\begin{figure}[t]
	\centering
	\resizebox*{14cm}{!}{\includegraphics{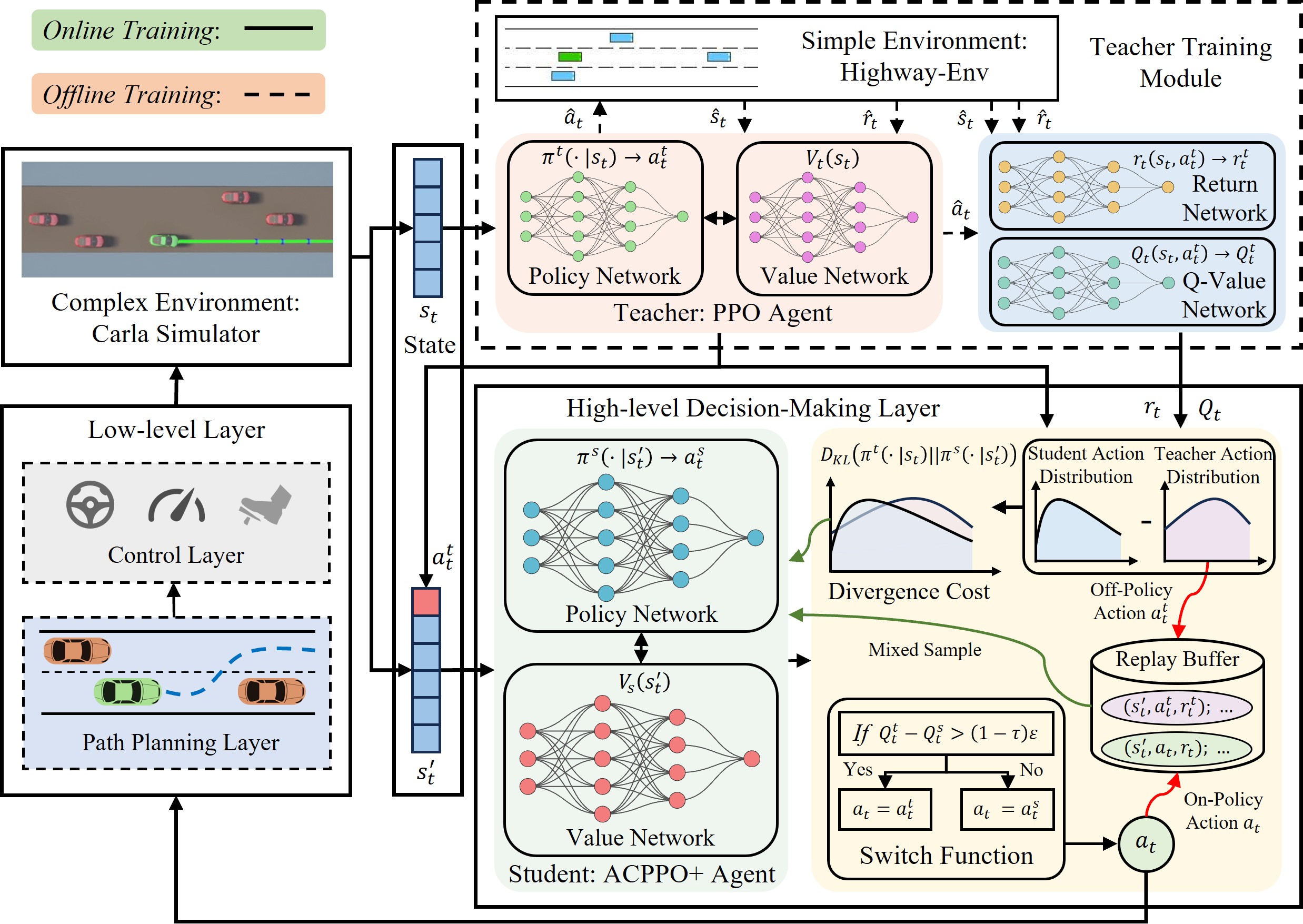}}
	\caption{
		The S2CD framework consists of a teacher training module, a high-level decision-making layer, and a lower-level layer, enhanced by 4 innovative modules to improve learning efficiency and performance:
		1. Employing the teacher model for action intervention and demonstration to enhance the safety of the student policy;
		2. Utilizing dual-source data for training to improve sample efficiency;
		3. Employing KL divergence constraints in policy updates to enable the student policy to quickly approach the teacher policy;
		4. Gradually reducing the teacher model's intervention via a weaning mechanism to prevent excessive reliance of the student on the teacher.
	}
	\label{fig:S2CD_framework}	
\end{figure}

\subsubsection{Teacher Training Module}
\label{sec:teacher_training_module}
The teacher training module functions as an offline pre-training module.
We pre-train the teacher model using the PPO algorithm in the lightweight autonomous driving simulation environment, Highway-env.
In Highway-env, due to the significant simplification of the decision-making and control processes for autonomous vehicles, an effective decision model can be easily trained using the PPO algorithm.

Furthermore, the data collected from the interaction between the agent and the environment, including states, actions, and returns, can be used not only to train the PPO algorithm but also to train the Return and Q-Value networks to approximate the state-action's return function and Q-value function.
When training the student agent in the Carla simulation environment, given a state $s_{t}$, the teacher model outputs a guiding action $a^{t}_{t}$ based on the probability distribution of $\pi^{t}$.
The Return and Q-Value networks then predict the short-term reward (Return) and the long-term reward (Q-value) for selecting action $a^{t}_{t}$ in state $s_{t}$.
Thus, the teacher training module can provide not only decision guidance for the current state but also precise quantitative feedback for the student agent.
The algorithm flow of the teacher training module is presented in \autoref{alg:teacher_model_training}.

\begin{algorithm}[t]
	\SetAlgoLined
	\caption{Teacher Model Training}
	\label{alg:teacher_model_training}
	\KwIn{Simple Environment State $\hat{s}_{t}$}
	\KwOut{Teacher Action $a^{t}_{t}$, Predicted Return Value $r^{t}_{t}$, and Q-value $Q^{t}_{t}$}
	\BlankLine
	Randomly initialize Actor network $\hat{NN}_{1}$ and Critic network $\hat{NN}_{2}$ with weights $\hat{\theta}_{1}$ and $\hat{\theta}_{2}$  \\
	Randomly initialize Return network $\bar{NN}_{1}$ and Q-Value network $\bar{NN}_{2}$ with weights $\bar{\varphi}_{1}$ and $\bar{\varphi}_{2}$  \\
	Copy Actor and Critic networks to Actor and Critic target networks respectively: $\hat{NN}_{1}' \leftarrow \hat{NN}_{1}, \hat{NN}_{2}' \leftarrow \hat{NN}_{2}, \hat{\theta}_{1}' \leftarrow \hat{\theta}_{1}$, $\hat{\theta}_{2}' \leftarrow \hat{\theta}_{2}$ \\	
	\For{$n=1:N_{total}$}{
		\For{$t=1:T_{epoch}$}{
			Get action from $\pi^{t}$: $\hat{a}_{t} \leftarrow \pi^{t}(\cdot|\hat{s}_{t})$ \\
			Adopt $\hat{a}_{t}$ and collect $(\hat{s}_{t}, \hat{a}_{t}, \hat{r}_{t}, \hat{s}_{t+1})$ into replay buffer $\hat{B}$\\
			Compute advantage estimates $\hat{A}^{\pi_{\theta}}$ \\
			Compute Q-Value of $\hat{a}_{t}$: $\hat{Q}_{t}$\\
			Collect $(\hat{s}_{t}, \hat{a}_{t}, \hat{r}_{t}, \hat{Q}_{t})$ into replay buffer $\bar{B}$\\
		}
		Update Actor network $\hat{NN}_{1}$ and Critic network $\hat{NN}_{2}$ using $\hat{B}$ and $\hat{A}^{\pi_{\theta}}$ by maximizing $L^{CLIP}(\theta)$: $\hat{\theta}_{1}^{new}, \hat{\theta}_{1}'^{new}, \hat{\theta}_{2}^{new}, \hat{\theta}_{2}'^{new} \leftarrow \hat{\theta}_{1}, \hat{\theta}_{1}', \hat{\theta}_{2}, \hat{\theta}_{2}'$ \\
		Update Return network $\bar{NN}_{1}$ and Q-Value network $\bar{NN}_{2}$ using $\bar{B}$ via gradient descent: $\bar{\varphi}_{1}^{new}, \bar{\varphi}_{2}^{new} \leftarrow \bar{\varphi}_{1}, \bar{\varphi}_{2}$
	}
\end{algorithm}

\subsubsection{High-level Decision-Making Layer}
In the high-level decision-making layer, several modifications have been made to the traditional PPO algorithm by incorporating innovative modules designed to enhance learning efficiency and performance.

\paragraph{\textbf{Action intervention and demonstration}}
First, after completing the offline training of the teacher model, the environment state $s_t$ is fed into the teacher's policy network $\pi^{t}$, which outputs the action $a^{t}_{t}$ with the highest probability ($a^{t}_{t}$ considered optimal by the teacher for $s_t$).
Simultaneously, based on the Return network and Q-Value network, the return and Q-value for action $a^{t}_{t}$ in the given state $s_t$ are predicted.
Next, the teacher’s guiding action $a^{t}_{t}$ is used as input to the student model, along with the state $s_t$, to form a new state representation $s_t'$.
The student agent’s policy network generates the probability distribution of actions based on $s_t'$ and selects the action $a^{s}_{t}$ with the highest probability.
At this stage, both the teacher and student models output actions $a^{t}_{t}$ and $a^{s}_{t}$, respectively, according to their own policies.
To enhance the teacher’s guidance of the student and allow for intervention in the student’s actions when necessary, we propose a novel switch function based on the Q-values of $a^{t}_{t}$ and $a^{s}_{t}$, as shown below:
\begin{equation}
	\label{eq:switch_function}
	a_t = \left\{
	\begin{array}{lr}
		a^{t}_{t}, & {if}\ Q^{t}_{t}-Q^{s}_{t}>(1-\tau)\varepsilon \\
		a^{s}_{t}, & otherwise \\
	\end{array}
	\right.
\end{equation}

According to this switch function, we can define the mixed behavior policy $\pi^{\text{mix}}$ as follows:

\begin{equation}
	\label{eq:mix_policy}
	\pi^{\text{mix}}(\cdot | s) = \mathcal{T}(s) \pi^{t}(\cdot | s) + \left( 1 - \mathcal{T}(s) \right) \pi^{s}(\cdot | s)
\end{equation}
\noindent
where $\mathcal{T}(s)=1$ indicates that the teacher has intervened in the action, and $\mathcal{T}(s)=0$ means otherwise.

The switch function indicates that the teacher will tolerate the student when the teacher cannot significantly outperform the student in terms of long-term returns.
Specifically, when the difference in Q-value between actions $a^{t}_{t}$ and $a^{s}_{t}$ exceeds the tolerance threshold $(1-\tau)\varepsilon$, it implies that the student’s performance is suboptimal, necessitating intervention by the teacher.
Otherwise, the teacher acknowledges the student’s performance and refrains from overriding the action.
Here, Q-values are predicted by the trained Q-Value network, $\varepsilon$ is the tolerance coefficient used to control the extent of the student’s reliance on the teacher, and $\tau$ is a decay coefficient, where $1-\tau$ increases the teacher’s tolerance, gradually reducing the student’s dependence on the teacher.
$\tau$ is shown in \autoref{eq:decay_ex}:
\begin{equation}
	\label{eq:decay_ex}
	\tau=\frac{1}{1+e^{\tfrac{n_{e}}{q_{1}}-q_{2}}}
\end{equation}
\noindent
where $n_{e}$ is the number of episodes experienced, and the values of $q_{1}$ and $q_{2}$ are shown in \autoref{tab:parameters_of_S2CD}.
Based on this switch function, the final action $a_{t}$ can be determined for execution.
\begin{table}
	\centering
	\caption{Parameters of S2CD framework}
	\label{tab:parameters_of_S2CD}
	\begin{tabular}{cc|cc}
		\toprule
		\textbf{Hyper-parameters} & \textbf{Value} & \textbf{Hyper-parameters} & \textbf{Value} \\ 
		\midrule
		Discount factor $\gamma$ & 0.96 & Learning rate & 0.0005 \\
		Learning rate decay & True & Lambda entropy $\beta$ & 0.01 \\ 
		Total training steps & 300K & Lambda advantage $\lambda$ & 0.98 \\ 
		Optimizer & AdamW & Hyperparameter $q_{1}$ & 3 \\
		Mini batch size & 64 & Hyperparameter $q_{2}$ & 10 \\
		 Clip parameter $\epsilon$ & 0.2 & Total steps per episode & 5,000\\ 
		Hyperparameter of AC $\psi$ & 0.2 & Initial Lagrange multiplier $\xi$ & 0.01 \\
		Efficiency weight $\alpha_1$ & 0.5 & Safety weight $\alpha_2$ & 1.0 \\
		Tolerance coefficient $\varepsilon$ & 0.5 & & \\
		\bottomrule
	\end{tabular}
\end{table}

To evaluate the effectiveness of the teacher’s policy intervention, we derive the upper and lower bounds of the return $J_{\theta_{\text{mix}}}$ for the mixed behavior policy, as shown in \autoref{the:effect_bound}:
\begin{theorem}
	\label{the:effect_bound}
	With the switch function, the return $J(\theta_\text{{mix}})$ is bounded both below and above by:
	\begin{equation}
		\label{eq:effect_bound}
		J(\theta_{t}) + \frac{\sqrt{2}(1 - \omega)R_{\max}}{(1 - \gamma)^2} \sqrt{H - \kappa}
		\geq J(\theta_\text{{mix}}) 
		\geq J(\theta_{t}) - \frac{\sqrt{2}(1 - \omega)R_{\max}}{(1 - \gamma)^2} \sqrt{H - \kappa}
	\end{equation}
	\noindent
	where $H = \mathbb{E}_{s \sim d_{\pi^{\text{mix}}}} \mathcal{H}(\pi^{t}(\cdot | s))$ represents the average entropy of the teacher policy, $\kappa$ is a small error term, $R_{\max}$ represents the maximum reward value, and $\omega$ is the intervention rate determined by the switch function.
\end{theorem}

Appendix \ref{Appendix:Theorem_3} presents a detailed derivation of \autoref{the:effect_bound}, which demonstrates that, via the switch function, the student can effectively improve the upper bound of performance through teacher interventions while maintaining a guaranteed lower bound.
Furthermore, the student's performance is strongly influenced by the quality of the teacher.

Simultaneously, to illustrate the improvement in safety performance of S2CD resulting from teacher policy intervention, we introduce \autoref{the:safety_bound} to demonstrate that the expected cumulative reward obtained by policy $\pi^{\text{mix}}$ is assured to be no less than the expected cumulative return gained by policy $\pi^{t}$, as expressed by:

\begin{theorem}
	\label{the:safety_bound}
	The expected cumulative reward acquired from $\pi^{\text{mix}}$ is ensured to be no less than the expected cumulative reward derived from $\pi^{t}$:
	\begin{equation}
		\label{eq:safety_bound}
		\mathbb{E}_{\pi^{\text{mix}}} \left[ \sum_{t=0}^{H} \gamma^t r(s_t, a_t) \right] 
		\geq \mathbb{E}_{\pi^{t}} \left[ \sum_{t=0}^{H} \gamma^t r(s_t, a_t) \right]
	\end{equation}
\end{theorem}
Appendix \ref{Appendix:Theorem_4} presents a detailed derivation of \autoref{the:safety_bound}, where \autoref{eq:safety_bound} shows that the S2CD framework can ensure enhanced driving safety performance through the teacher's guidance.

\paragraph{\textbf{Model training with dual-source data}}
When action $a_{t}$ is executed, the environment provides feedback, such as the return, to the student agent for updating the neural network.
At this point, the student agent has two complete sets of data available for training and optimizing the policy $\pi^{s}$: $(s_{t}', a_{t}, r_{t})$ consists of on-policy data collected by executing action $a_{t}$ in the current state $s_{t}'$ and receiving environmental feedback $r_{t})$, while the other set $(s_{t}', a^{t}_{t}, r^{t}_{t})$ is off-policy data provided by the teacher module, where action $a^{t}_{t}$ and the return $r^{t}_{t}$ are predicted based on state $s_{t}$.
Due to the additional data provided by the teacher model, the S2CD framework achieves significantly higher sample efficiency.
For the two distinct data sources from the teacher and student, the teacher's guidance is particularly crucial, particularly in the initial phases of training.
Therefore, we applied the ACPPO+ algorithm proposed in \autoref{sec:ACPPO+} to optimize and update the neural network, to further enhance the learning efficiency of the algorithm.
ACPPO+ improves learning efficiency by loosening the update magnitude for significant samples while tightening it for less important ones.

\paragraph{\textbf{KL divergence constraint}}
In addition, to enable the student policy to quickly approximate the distribution of the teacher policy, we adopted a behavior cloning mechanism.
The KL divergence between the teacher's and student's policies is used as a constraint in the algorithm, and the objective function is reformulated as an unconstrained optimization problem using the Lagrangian method. 
The constraint function is as follows:
\begin{equation}
	\label{eq:dkl_cost}
	C^{KL}_t = D_{KL}\left (\pi^{t}(\cdot|s_{t})|| (\pi^{s}(\cdot|s_{t}'))
	\right.
\end{equation}

So the objective function can be defined as follows:
\begin{equation}
	\begin{split}
		L^{ACLIP+'}(\theta) &= \mathbb{E}_{\pi_\theta} \Bigg[ \mathbb{I}_{a_t^s} \cdot \min \left( r_t(\theta) A^{\pi_\theta}(s, a_t^s), \text{clip}\left(r_t(\theta), 1 - (\epsilon + \epsilon'), 1 + (\epsilon - \epsilon')\right) A^{\pi_\theta}(s, a_t^s) \right) \\ 
		& \quad + \mathbb{I}_{a_t^t} \cdot \min \left( r_t(\theta) A^{\pi_\theta}(s, a_t^t), \text{clip}\left(r_t(\theta), 1 - (\epsilon - \epsilon'), 1 + (\epsilon + \epsilon')\right) A^{\pi_\theta}(s, a_t^t) \right) - \xi C^{KL}_t \Bigg]
	\end{split}
\end{equation}
\noindent
where $\xi$ represents the Lagrange multiplier for $C^{KL}_t$, and its initial value is provided in the \autoref{tab:parameters_of_S2CD}.

\paragraph{\textbf{Intervention decay}}
The teacher’s intervention can significantly accelerate the student’s learning process and prevent dangerous actions, thereby greatly enhancing the student’s model performance, especially during the early stages of training.
However, since the teacher’s policy has an upper performance limit, over-reliance on the teacher’s guidance can constrain the student’s ultimate performance.
To address this issue, we introduced a weaning mechanism that gradually reduces the teacher’s intervention in the student during the mid-to-late stages of training, thereby mitigating the influence of the teacher’s policy.
This allows the S2CD framework to effectively balance teacher guidance with student exploration, ultimately achieving superior performance compared to a single policy.
The weaning mechanism is composed of two main components:
\begin{itemize}
	\item
	We use the decay coefficient $\tau$, which gradually increases the tolerance of the teacher’s policy toward the student.
	Specifically, the tolerance increases by $1-\tau$, leading to a progressive reduction in the intervention probability $\omega$ of the teacher’s actions, ultimately decreasing it to 0.
	\item
	Secondly, as the teacher's intervention probability progressively decreases, we utilize $\tau$ to gradually reduce both the adaptive factor $\epsilon'$ and the constraint function $C^{KL}_{t}$, thereby diminishing the teacher's influence on the student.
\end{itemize}

Thus, the final objective function of the S2CD framework is defined as follows:
\begin{equation}
	\label{eq:S2CD_obj}
	\begin{split}
		L^{S2CD}(\theta) &= \mathbb{E}_{\pi_\theta} \Bigg[ \mathbb{I}_{a_t^s} \cdot \min \left( r_t(\theta) A^{\pi_\theta}(s, a_t^s), \text{clip}\left(r_t(\theta), 1 - (\epsilon + \tau\epsilon'), 1 + (\epsilon - \tau\epsilon')\right) A^{\pi_\theta}(s, a_t^s) \right) \\ 
		& \quad + \mathbb{I}_{a_t^t} \cdot \min \left( r_t(\theta) A^{\pi_\theta}(s, a_t^t), \text{clip}\left(r_t(\theta), 1 - (\epsilon - \tau\epsilon'),1 + (\epsilon + \tau\epsilon')\right) A^{\pi_\theta}(s, a_t^t) \right) - \tau \xi C^{KL}_t \Bigg]
	\end{split}
\end{equation}

According to $\tau$, we allow the teacher to provide as many correct decisions as possible during the early stages of training to intervene and guide the student's learning.
As training progresses, teacher intervention is gradually reduced in the later stages, allowing the student to independently explore better strategies.
Eventually, as $\tau$ decreases and approaches zero, $L^{S2CD}(\theta)$ gradually returns to the form of \autoref{eq:ppo_obj}.
The operation of the S2CD decision-making framework is shown in \autoref{alg:S2CD_algorithm}, and all parameters are listed in \autoref{tab:parameters_of_S2CD}.

\begin{algorithm}[t]
	\SetAlgoLined
	\caption{S2CD Decision-Making Framework}
	\label{alg:S2CD_algorithm}
	\KwIn{Complex Environment State $s_{t}$, Teacher Action $a^{t}_{t}$, Predicted Return Value $r^{t}_{t}$, Q-value $Q^{t}_{t}$}
	\KwOut{Ego Vehicle Action $a_{t}$} 
	\BlankLine
	Randomly initialize Actor network $NN_{1}$ and Critic network $NN_{2}$ with weights $\theta_{1}$ and $\theta_{2}$  \\
	Copy Actor and Critic networks to Actor and Critic target networks respectively: $NN_{1}' \leftarrow NN_{1}, NN_{2}' \leftarrow NN_{2}, \theta_{1}' \leftarrow \theta_{1}, \theta_{2}' \leftarrow \theta_{2}$ \\
	\For{$n=1:N_{total}$}{
		\For{$t=1:T_{epoch}$}{
			Get action from teacher policy $\pi^{t}$: $a^{t}_{t} \leftarrow \pi^{t}(\cdot|s_{t})$ \\
			Add $a^{t}_{t}$ to $s_{t}$: $s_{t}' \leftarrow s_{t}$ \\
			Get action from student policy $\pi^{s}$: $a^{s}_{t} \leftarrow \pi^{s}(\cdot|s_{t}')$ \\
			Compute advantage estimates $A^{\pi_{\theta}}$ \\
			Predict Q-values of $a^{s}_{t}$ and $a^{t}_{t}$: $Q^{s}_{t}, Q^{t}_{t} \leftarrow$ Q-Value network \\
			\eIf{$Q^{t}_{t} - Q^{s}_{t} > (1-\tau)\varepsilon$}{
				Choose action under the guidance of teacher model: $a_{t} \leftarrow a^{t}_{t}$ \\
			}{
				Choose action based on Actor network: $a_{t} \leftarrow a^{s}_{t}$ \\
			}
			Adopt $a_{t}$ and collect $(s_{t}', a_{t}, r_{t}, s_{t+1}')$ into replay buffer $B$ \\
			Predict return value of $a^{t}_{t}$: $r^{t}_{t} \leftarrow$ Return network \\
			Collect $(s_{t}', a^{t}_{t}, r^{t}_{t}, s_{t+1}')$ into replay buffer $B$ \\
			Compute Kullback-Leibler divergence of teacher and student policies: $D_{KL} (\pi^{t}(\cdot|s_{t})|| \pi^{s}(\cdot|s_{t}'))$ \\
			Compute adaptive clipping factor $\epsilon'$ \\
			Update decay coefficient $\tau$
		}
		Update Actor network $NN_{1}$ and Critic network $NN_{2}$ using $B$, $A^{\pi_{\theta}}$ and $D_{KL}$ by maximizing $L^{S2CD}(\theta)$: $\theta_{1}^{new}, \theta_{1}'^{new}, \theta_{2}^{new}, \theta_{2}'^{new} \leftarrow \theta_{1}, \theta_{1}', \theta_{2}, \theta_{2}'$
	}
\end{algorithm}

\subsubsection{Low-level Layer}
\label{sec:low_level_layer}
In the S2CD framework, we developed a hierarchical architecture to enable autonomous vehicles to drive smoothly.
This framework consists of two components: high-level decision-making and low-level execution.
First, the high-level layer outputs lane-change or car-following commands according to the current state, providing a macro driving strategy for the vehicle.
Next, the low-level execution layer receives these high-level commands and is responsible for path planning, trajectory and speed control to precisely carry out the selected driving strategy.

\paragraph{\textbf{Path planning layer}}
We represent the centerline of the path with a series of points $(x_{0},y_{0}), (x_{1},y_{1}),\dots ,(x_{n},y_{n})$, where each segment $S_{i}$ represents a portion along the centerline, bounded by points $(x_{i},y_{i})$ and $(x_{i+1},y_{i+1})$.
The vehicle’s driving path is then planned using a cubic spline function, defined as follows:
\begin{equation}
	\label{eq:cubic_ex}
	y_{i}(x)=a_{i}+b_{i}(x-x_{i})+c_{i}(x-x_{i})^{2}+d_{i}(x-x_{i})^{3}
\end{equation}
\noindent
where $a_{i}$, $b_{i}$, $c_{i}$, and $d_{i}$ are parameters determined by the approach outlined in \citet{jiang2020dynamic}.

The ego vehicle’s global trajectory is determined by the built-in path navigation of the Carla simulator.
When the decision-making layer issues a lane-change command, the global navigation point is positioned in the target lane at a distance of 10 meters from the vehicle's current position.
\autoref{eq:cubic_ex} is then used to determine the trajectory of the lane-change segment for local path planning.

\paragraph{\textbf{Control layer}}
Since this study focuses on lane-change and car-following decision-making algorithms, the acceleration and speed of the ego vehicle are controlled using the Intelligent Driver Model (IDM) \citep{treiber2000congested} to ensure simplicity and prevent overly conservative behavior by the agent.
Therefore, the target acceleration $\dot{v}_{e}$ is calculated using \autoref{eq:IDM}.
\begin{equation}
	\label{eq:IDM}
	\dot{v}_{e}=a \left[1-\left(\dfrac{v_{e}}{v_{0}} \right)^{\zeta}- \left(\dfrac{max \left(s_{0}+v_{e}T+\dfrac{v_{e}(v_{e}-v_{1})}{2\sqrt{ab}}, 0 \right)}{s} \right)^{2} \right]
\end{equation}
\noindent
where $v_{e}$ and $v_{1}$ represent the speeds of the ego and the lead vehicles, respectively, while other parameters in the \ref{eq:IDM} are detailed in \autoref{tab:IDM_parameters}.
Using $v_{e}$ and $\dot{v}_{e}$, the target speed for the next moment can be calculated.

After establishing the target speed and driving path, the PID algorithm is applied to determine the specific control values.
To enhance control performance, two sets of parameters are utilized, as detailed in \autoref{tab:PID_parameters}.
\begin{table}
	\centering
	\begin{minipage}[b]{0.4\linewidth}
		\centering
		\caption{Parameters of IDM model}
		{\begin{tabular}{cc} 
				\toprule
				\textbf{Parameters} & \textbf{Value} \\
				\midrule
				Acceleration exponent $\zeta$ & 4 \\
				Target speed $v_{0}$ & 25 m/s \\
				Target time gap $T$ & 0.6 s \\
				Safety distance $s_{0}$ & 2.0 m \\
				Maximal acceleration $a$ & $2 m/s^{2}$ \\
				Comfortable deceleration $b$ & $2 m/s^{2}$ \\
				\bottomrule
		\end{tabular}}
		\label{tab:IDM_parameters}
	\end{minipage}
	\hspace{0.1cm}
	\hfill
	\begin{minipage}[b]{0.55\linewidth}
		\centering
		\caption{Parameters of PID controller}
		{\begin{tabular}{cc} 
				\toprule
				\textbf{Parameters} & \textbf{Value} \\ 
				\midrule
				Proportional Gain of lateral control $K_{P}$ & 0.75 \\
				Derivative Gain of lateral control $K_{D}$ & 0.01 \\
				Integral Gain of lateral control $K_{I}$ & 0.2 \\
				Proportional Gain of longitudinal control $K_{P}'$ & 0.37 \\
				Derivative Gain of longitudinal control $K_{D}'$ & 0.012 \\
				Integral Gain of longitudinal control $K_{I}'$ & 0.016 \\
				\bottomrule
		\end{tabular}}
		\label{tab:PID_parameters}
	\end{minipage}
\end{table}

\section{Implementation}
\label{sec:implementation}
This section begins by comparing the differences between the Highway-Env and Carla simulation environments, followed by a detailed description of a medium-density highway lane-change scenario.
Next, the scenario is modeled as a Markov Decision Process (MDP), and finally, all settings used in the simulation experiments are introduced.

\subsection{Highway-Env and Carla Simulation Environments}
Highway-Env and Carla are widely used simulation environments in autonomous driving research.
However, they differ significantly in terms of the complexity of decision-making processes, the level of control detail, and computational resource requirements.
Highway-Env represents the environment as a simplified 2-dimensional plane and reduces vehicle actions to discrete choices, such as accelerating, decelerating, and lane changing.
Consequently, the decision-making process in this environment is relatively straightforward, typically requiring only tens to hundreds of simulation steps to complete a task.
Because of the simplified environment, the simulation avoids complex perception data processing.
Additionally, the vehicle dynamics model is highly abstracted, ignoring real-world physical characteristics.
As a result, Highway-Env is extremely resource-efficient, achieving simulation speeds of hundreds or even thousands of iterations per second on standard computing devices.
This makes it ideal for large-scale RL training and strategy validation.

In contrast, Carla provides a high-fidelity 3-dimensional simulation environment that can model complex traffic scenarios and support the processing of various sensor data (e.g., cameras, LiDAR), along with continuous vehicle control.
The complexity of the environment and vehicle dynamics necessitates that the decision-making process consider multi-dimensional continuous control variables, such as steering angle, throttle, and brake force.
A typical driving task in Carla may require thousands to tens of thousands of simulation steps to complete.
Additionally, Carla requires higher simulation precision, involving high-dimensional sensor data processing and detailed physical simulations, which typically necessitates high-performance computing devices support.
Consequently, Carla’s simulation speed is relatively slow, achieving only a few to tens of simulation steps per second.
This makes Carla more suitable for testing and validating autonomous driving systems in complex environments, especially when real-world perception, planning, and control challenges are the focus.

In summary, Highway-Env is well-suited for the rapid validation of simple decision-making processes, particularly in RL strategy research, due to its simplified environment and lower computational requirements.
In contrast, Carla provides a more realistic testing platform for developing and validating complex autonomous driving systems.
The differences between the two environments are not only reflected in the simplification of environment representation and action space but also in the significant disparity in the number of decision steps required and the computational resources consumed by the simulations.

\subsection{Scenarios Specification}
\label{sec:scenarios_specification}
Medium traffic density on highways is defined as having roughly 15 vehicles per lane per kilometer, with vehicle spacing between 50 and 90 meters \citep{gold2016taking}.
This density level provides a balance by avoiding both congested conditions, which make lane changes difficult, and low-density conditions, where lane changes are largely unnecessary.
Therefore, it represents a typical scenario for most lane-change situations.
We use this scenario as the platform for training and evaluating the performance of the S2CD framework.
\autoref{fig:Traffic_scenario} shows the traffic scenario on the three-lane highway in the Highway-Env and Carla simulation environments, where the ego vehicle performs a lane change while surrounded by other vehicles.
According to driving principles and relevant traffic regulations, the ego vehicle must consider the positions and movements of the front, front-left, rear-left, front-right, and rear-right vehicles when executing a lane change.
Therefore, the states of the ego vehicle and surrounding vehicles, including their speeds and distances, will serve as inputs to the model.

In the experiments, all vehicles start with an initial velocity of 0 m/s, with a maximum velocity of 25 m/s.
To more accurately simulate real traffic conditions, the target speeds of vehicles are set between 15 and 25 m/s.
If the ego vehicle makes improper lane-change decisions while driving, it may come too close to other vehicles or even collide with them or the road boundary.
If the ego vehicle successfully travels 1 kilometer without any collisions, the episode is considered successful.
In the event of a collision, the episode is marked as a failure, and the environment is immediately reset for the next episode.

\begin{figure}[t]
	\centering
	\subcaptionbox{Highway-Env simulator\label{fig:Highway_Env_simulator}}{
		\resizebox*{7.5cm}{!}{\includegraphics{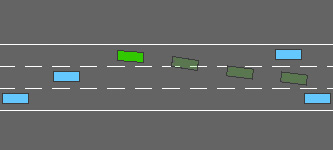}}}
	\hspace{0.5cm}
	\subcaptionbox{Carla simulator\label{fig:Carla_simulator}}{
		\resizebox*{7.5cm}{!}{\includegraphics{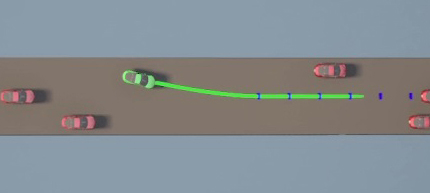}}}
	\caption{Medium-density traffic scenario with 3-lanes}
	\label{fig:Traffic_scenario}
\end{figure}

\subsection{Scenario Modeling}
Next, we define the state space, action space, and return function, thereby modeling the driving scenario as a MDP.

\subsubsection{State}
The state $s_{t}$ includes information of vehicles.
As described in \autoref{sec:scenarios_specification}, $s_{t}$ is as follows:
\begin{equation}
	\label{eq:st_ex}
	s_{t} = (v_{e,t},\ \{v_{i,t},\ d_{i,t}\})
\end{equation}
\noindent
where $v_{e}$ is the velocity of the ego vehicle; $\{v_{i},\ d_{i}\}$ represents the velocities and distances of surrounding vehicles, where $i=(1,2,3,4,5)$ corresponds to the front, front-left, rear-left, front-right, and rear-right vehicles, respectively.

\subsubsection{Action}
In this paper, the decision-making layer provides discrete commands for the low-level execution layer.
Consequently, the action $a_{t}$ is defined as follows:
\begin{equation}
	\label{eq:action_ex}
	a_{t} = (a_{1},\ a_{2},\ a_{3})
\end{equation}
\noindent
where $a_1$,$a_2$, and $a_3$ represent Follow, Left Lane Change, and Right Lane Change, respectively.

\subsubsection{Return}
In the highway lane-change scenario, the main goal of autonomous vehicles is to drive safely and at high speed.
Therefore, the return function is designed with the following two components:
\begin{enumerate}	
	\item
	\textbf{Efficiency Reward}:
	The maximum velocity for all vehicles is capped at 25 m/s.
	To prompt the ego vehicle to achieve greater speeds, we adopted the following efficiency reward function:
	\begin{equation}
		\label{eq:Re_ex}
		R_{e}=\left\{
		\begin{array}{lr}
			0,&{if}\ 0 <v_{e}< 12.5\ \text{m/s} \\ 
			\alpha_{1} \left(\dfrac{v_{e}}{12.5}-1 \right),&{if}\ 12.5\ \text{m/s}<v_{e}<25\ \text{m/s} \\
			\alpha_{1},&{otherwise}
		\end{array}
		\right.
	\end{equation}	
	\noindent
	where $\alpha_1 > 0$ is a hyperparameter that adjusts the significance of efficiency within the policy.
	\item
	\textbf{Safety Cost}:
	To maintain a safe distance between the ego vehicle with surrounding vehicles and to prevent collisions, we define the safety cost function as follows:
	\begin{equation}
		\label{eq:Rs_ex}
		C_{s}=\left\{
		\begin{array}{lr}
			\alpha_{2},&{if}\ 0<d_{safe}<5\ \text{m} \\
			\alpha_{2} \left(1-\dfrac{d_{safe}-5}{5} \right),&{if}\ 5\ \text{m}<d_{safe}<10\ \text{m} \\
			1,&{if}\ collision \\
			0,&{otherwise}
		\end{array}
		\right.
	\end{equation}
	\noindent
	where $\alpha_2 > 0$ is a hyperparameter that adjusts the significance of safety within the policy; $d_{safe}$ represents the safe distance, defined as the distance between the ego vehicle and the vehicle directly in front or behind.
\end{enumerate}

The values of $\alpha_1$ and $\alpha_2$ are provided in \autoref{tab:parameters_of_S2CD}, and the final Return function $R$ is presented in \autoref{eq:return}:
\begin{equation}
	\label{eq:return}
	R=R_{e}-C_{s}
\end{equation}

\subsection{Simulation Settings}
All simulation experiments were conducted on a computer equipped with an Intel i5-13600KF CPU, an NVIDIA GeForce RTX 4090 GPU, and 32GB of RAM.
The experiments involved two simulation environments, Highway-Env and Carla, with their respective runtime parameters listed in \autoref{tab:Highway-Env_parameters} and \autoref{tab:Carla_parameters}.
Given the focus of this paper on decision-making approaches, we streamlined the perception module by assuming the ego vehicle has direct access to accurate information (e.g., speed, distance) about surrounding vehicles within a 50-meter range.
To demonstrate the algorithm's generalization ability, different random seeds were employed each time the environment was reset.
In the Carla environment, the car-following and lane-changing behaviors of other vehicles are managed by the Traffic Manager.
In the Highway-Env simulation environment, other vehicles were managed using simple decision-making algorithms (e.g., IDM, MOBIL) and control algorithms (e.g., proportional controllers).

Since Highway-Env is a simplified simulation environment, while Carla is a high-fidelity simulation environment, simulating the same 1 km drive requires nearly 1,500 steps in Carla, compared to fewer than 150 steps in Highway-Env.
Specifically, when the ego vehicle makes a lane-change decision, the policy frequency in Highway-Env is set to 2, allowing the lane-change process to be completed in only 2 steps.
This simplification facilitates the training of effective teacher models using RL algorithms.
In contrast, the lane-change process in Carla requires 10 to 20 steps.
Consequently, obtaining an effective lane-change decision model using standard RL algorithms in Carla is challenging.

\begin{table}
	\centering
	\begin{minipage}[b]{0.45\linewidth}
		\centering
		\caption{Parameters of Highway-Env simulation}
		\label{tab:Highway-Env_parameters}
		{\begin{tabular}{cc} 
				\toprule
				\textbf{Parameters} & \textbf{Value} \\
				\midrule
				Lanes count & 3 \\
				Lane width & 3.75m \\
				Vehicle model & Kinematic \\
				Simulation frequency & 10Hz \\
				Policy frequency & 2Hz \\
				Auto lane change & True \\
				Speed limit & 25m/s \\
				\bottomrule
		\end{tabular}}
	\end{minipage}
	\hspace{0.1cm}
	\hfill
	\begin{minipage}[b]{0.45\linewidth}
		\centering
		\caption{Parameters of Carla simulation}
		\label{tab:Carla_parameters}
		{\begin{tabular}{cc} 
				\toprule
				\textbf{Parameters} & \textbf{Value} \\ 
				\midrule
				Lanes count & 3 \\
				Lane width & 3.75m \\
				Vehicle model & Dynamic \\
				Synchronous mode & True \\
				Simulation time step & 0.05s \\
				Auto lane change for other vehicle & True \\
				Speed limit & 25m/s \\
				\bottomrule
		\end{tabular}}
	\end{minipage}
\end{table}

\section{Simulation Result}
\label{sec:simulation_result}
We conducted performance testing and comparisons of the proposed S2CD against other state-of-the-art (SOTA) methods in a highway lane-change scenario:

DQN (Deep Q-Network \citep{mnih2013playing}, Value-based RL): DQN integrates deep learning with Q-learning by using a neural network to approximate Q-values, significantly improving performance in high-dimensional discrete state tasks.

PPO (Proximal Policy Optimization \citep{schulman2017proximal}, On-Policy RL): PPO employs clipped policy updates to ensure stability and efficiency in optimization, widely applied in both continuous and discrete control tasks.

SAC (Soft Actor-Critic \citep{haarnoja2018soft}, Off-Policy RL): SAC maximizes policy entropy to encourage exploration, significantly enhancing sample efficiency and performance in continuous action spaces.

PPO-Lag (PPO with Lagrangian Multipliers \citep{stooke2020responsive}, Safe RL): PPO-Lag extends PPO by introducing Lagrangian multipliers to handle constrained optimization, ensuring policy updates adhere to constraints while optimizing the objective.

SAC-Lag (SAC with Lagrangian Multipliers \citep{ha2021learning}, Safe RL): SAC-Lag extends SAC by incorporating Lagrangian multipliers to balance policy entropy with constraint satisfaction, making it suitable for constrained RL.

TS2C (Teacher-Student Shared Control \citep{xue2023guarded}, TSF): TS2C is a student-teacher framework where a pre-trained teacher agent guides the student agent by intervening and providing online demonstrations based on trajectory-based values, and we re-implemented this algorithm using PPO as the base method.

SOAR-ACPPO (SOAR Cognitive Architecture with Adaptive Clipping PPO \citep{zhou2024hybrid}, TSF): SOAR-ACPPO is a hybrid decision-making framework based on the TSF, where the SOAR cognitive decision model, incorporating human driving knowledge, serves as the teacher to effectively guide the student agent, which is based on the ACPPO algorithm.

BC (Behavioral Cloning \citep{pomerleau1991efficient}, IL): BC is a supervised imitation learning algorithm that directly learns policies by mimicking expert actions from example data.

GAIL (Generative Adversarial Imitation Learning \citep{ho2016generative}, IL): GAIL combines the generative adversarial framework with imitation learning, using a discriminator to distinguish between expert and learner policies, guiding the learner towards expert-level performance.

CQL (Conservative Q-Learning \citep{kumar2020conservative}, Offline RL): CQL is an offline reinforcement learning algorithm that penalizes overestimation of Q-values, ensuring
conservative policy updates, which is particularly effective when learning from static datasets.

In Safe RL algorithms, we introduce safety cost as constraints to regulate policy behavior, ensuring that the agent maximizes efficiency reward while adhering to safety constraints.
The specific parameters of the baseline algorithms are provided in Appendix \ref{Appendix:Hyper-parameters}.
Our analysis compares two aspects: model training and evaluation.

\subsection{Model Training}
\autoref{fig:model_training} presents the test results of various algorithms during the model training process.
The metrics presented include the return, safety cost, average speed from test episodes, as well as efficiency reward, safety cost, and collision counts generated from the agent's interactions with the environment every 5,000 steps.
Notably, CQL, as an offline RL algorithm, does not depend on interactions with the environment for data collection during training.
Therefore, CQL has been excluded from this comparison.

\begin{figure}[t]
	\centering
	\subcaptionbox{Test Return\label{fig:test_return}}{
		\resizebox*{5.2cm}{!}{\includegraphics{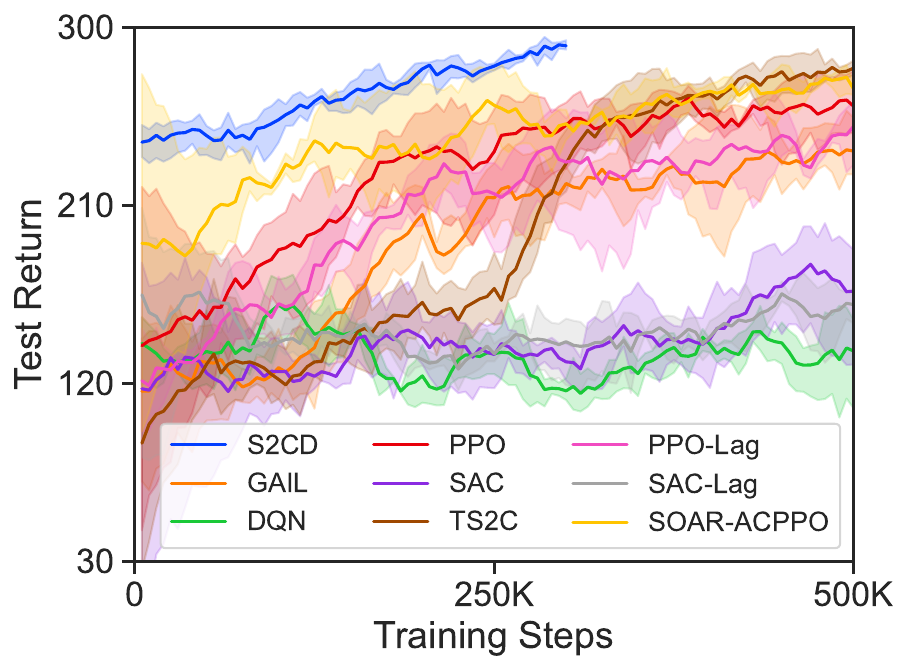}}}
	\subcaptionbox{Test Safety Cost\label{fig:test_safty_cost}}{
		\resizebox*{5.2cm}{!}{\includegraphics{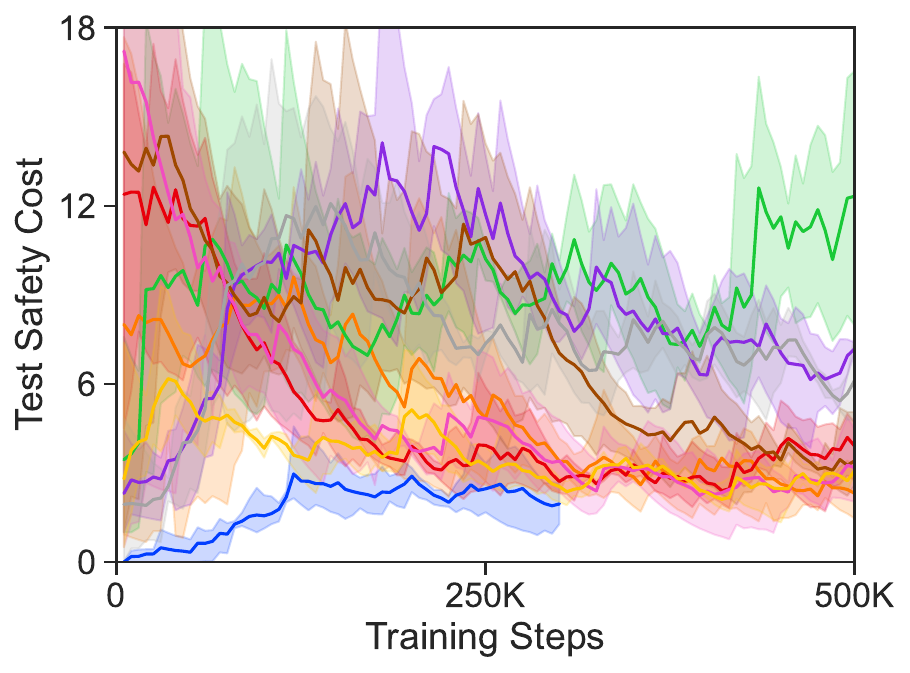}}}
	\subcaptionbox{Test Speed\label{fig:test_speed}}{
		\resizebox*{5.2cm}{!}{\includegraphics{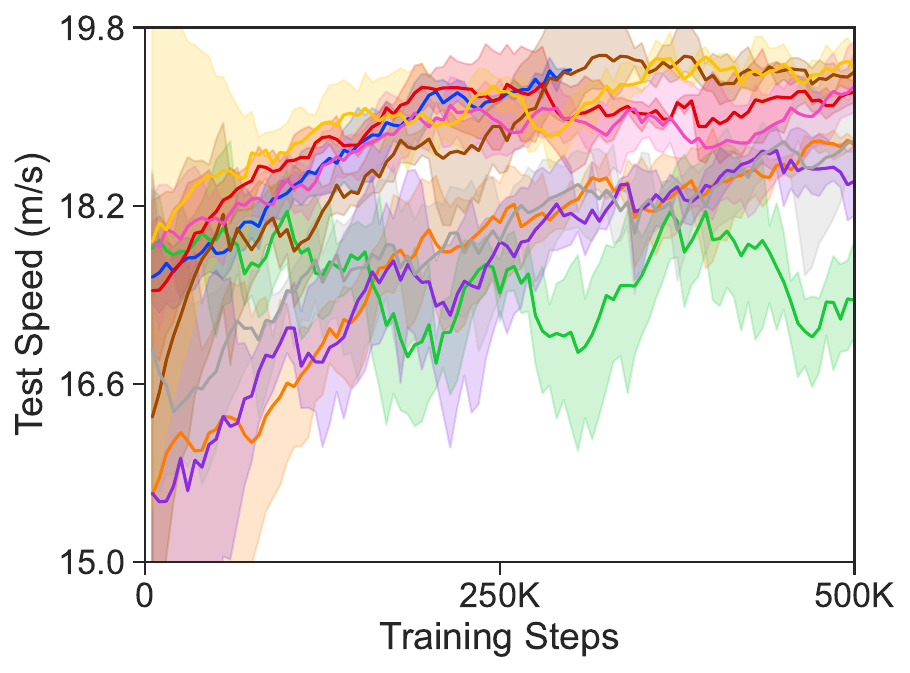}}}
	\subcaptionbox{Training Efficiency Reward\label{fig:training_efficiency_reward}}{
		\resizebox*{5.2cm}{!}{\includegraphics{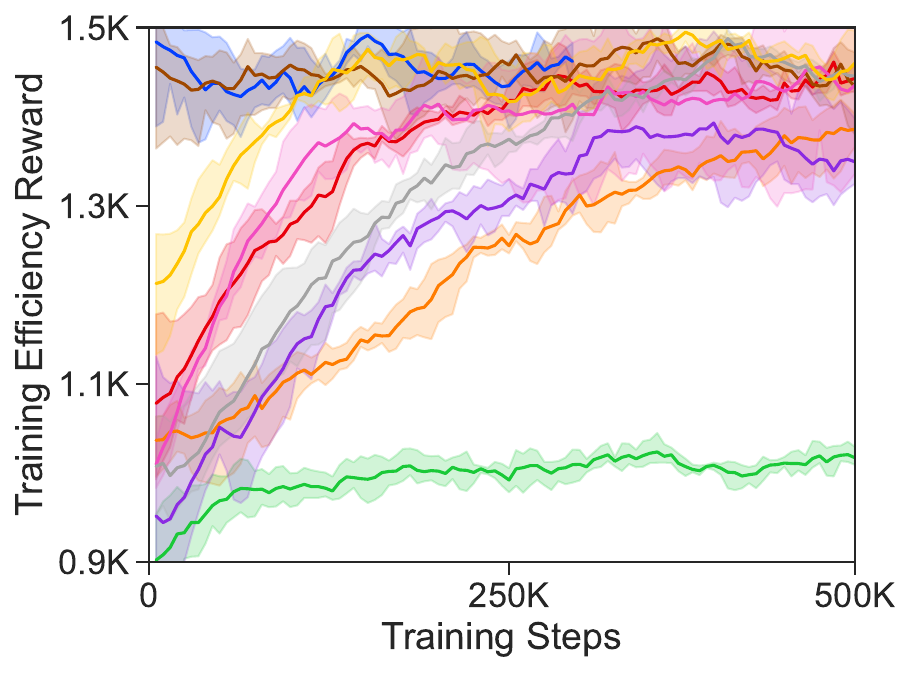}}}
	\subcaptionbox{Training Safety Cost\label{fig:training_safety_cost}}{
		\resizebox*{5.2cm}{!}{\includegraphics{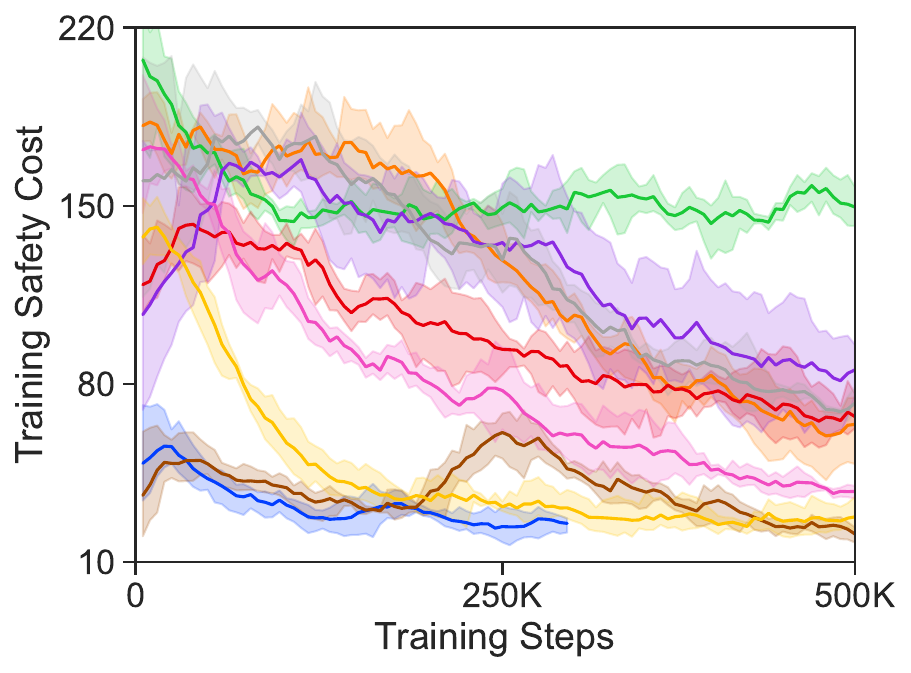}}}
	\subcaptionbox{Training Collision Counts\label{fig:training_collision_counts}}{
		\resizebox*{5.2cm}{!}{\includegraphics{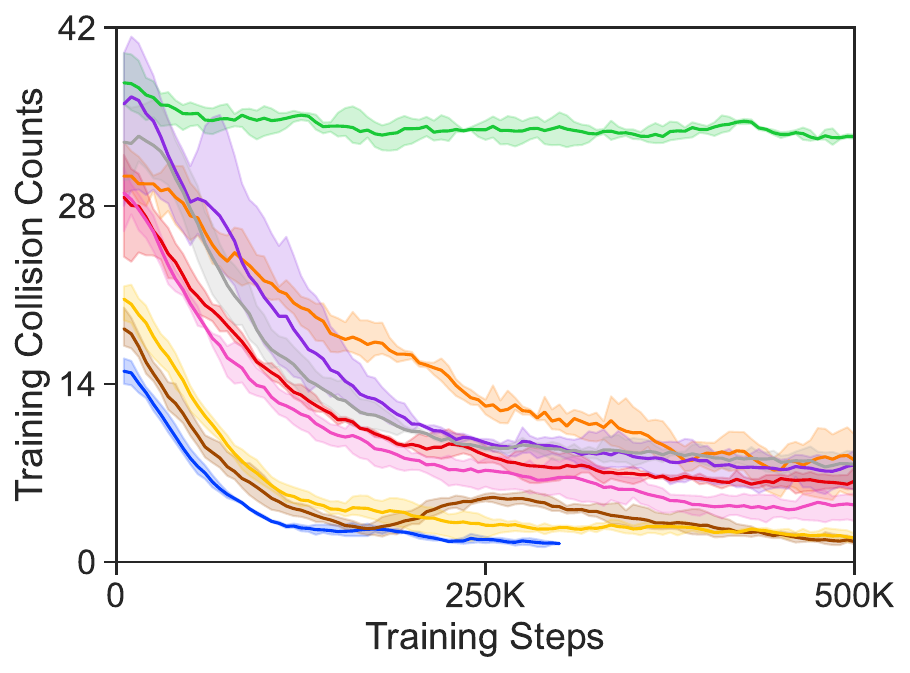}}}
	
	\caption{
		The training curves for all algorithms are shown.
		Every algorithm is trained using three random seeds, accumulating 500K training steps in every case, except for S2CD, which needs only 300K steps.
		To evaluate the performance of the algorithms at each stage, two evaluation episodes are conducted every 5,000 training steps, and the average value of these episodes is recorded as the result.
	}
	\label{fig:model_training}
\end{figure}

Throughout the training process of the algorithms, we primarily focus on the trends of key indicators, such as model performance and safety.
\autoref{fig:test_return} illustrates the test return curve as a function of training steps, clearly reflecting the overall performance of all algorithms.
As the curve indicates, guided by the teacher model, 
S2CD benefits from higher sample efficiency and learning efficiency, enabling it to acquire more accurate knowledge in the early stages.
This advantage allows it to cover longer distances without collisions, thereby achieving consistently high return values of approximately 240 from the outset.
For the TS2C algorithm, excessive teacher intervention during the early stages leads to an imbalance between positive and negative samples, and the discrepancy between the policy distributions of the teacher and the student hampers the student's ability to optimize its policy, resulting in a lower Return.
SOAR-ACPPO decides whether to intervene based on intervention probability, which prevents it from strictly following the teacher's guidance, consequently leading to a lower initial score. 
Nevertheless, SOAR-ACPPO's learning speed is significantly faster than that of other algorithms, except for S2CD.
In contrast, other RL algorithms gradually optimize their policies through environmental interactions.
However, due to limited early-stage knowledge, they experienced more frequent collisions during testing, resulting in significantly lower return values compared to S2CD.
Furthermore, \autoref{fig:training_efficiency_reward} demonstrates that the efficiency reward obtained by the S2CD and TS2C agents through interactions with the environment are considerably higher than those achieved by other algorithms, indicating that the teacher's intervention effectively enhances the student's ability to achieve high scores.
Regarding safety, \autoref{fig:training_safety_cost} illustrates the safety cost curves during the data collection period.
Due to the teacher model’s intervention in correcting the student’s unsafe actions, S2CD and TS2C guarantee that the ego vehicle keeps a safe distance from other vehicles whenever feasible.
Consequently, this leads to their safety cost during the data collection period being significantly lower than that of other algorithms.
However, as depicted in \autoref{fig:test_safty_cost}, the TS2C algorithm demonstrates poor performance in terms of test safety cost, indicating the model's inefficiency in learning.
In contrast, S2CD consistently achieves the lowest test safety cost.
\autoref{fig:training_collision_counts} illustrates that during the early data collection phase, S2CD experienced the least number of collision events.
Additionally, TS2C and SOAR-ACPPO, both benefiting from teacher guidance, also exhibited a similarly low collision probability, approximately half that of other algorithms.
Therefore, even in the initial stages of training, when the policy is not yet fully optimized, the teacher's guidance enables S2CD’s student agent to achieve performance and safety significantly superior to those of other algorithms.

During the mid-to-late stages of training, as agents accumulate data, and update their policies, the performance of other algorithms, excluding DQN, begins to improve.
Specifically, return, efficiency reward, and driving speed gradually increase, while safety cost and collision counts decrease.
Notably, the PPO algorithm and its derivatives: S2CD, TS2C, SOAR-ACPPO, PPO-Lag, and GAIL, demonstrate significantly faster performance improvements, whereas SAC and its derivative SAC-Lag learn at a slower pace.
This is because PPO-based algorithms limit the magnitude of policy updates, preventing substantial performance degradation in each iteration and ensuring that the policies steadily progress toward better solutions until convergence.
Among these, TSF-based algorithms (S2CD, TS2C, SOAR-ACPPO) demonstrated relatively superior final performance.
In contrast, because SAC-based algorithms are inherently designed for continuous action tasks, achieving satisfactory training results for the discrete action tasks in this study is challenging.
In terms of training efficiency, S2CD's enhancements in sampling and learning mechanisms result in a noticeably faster convergence speed, achieving a return greater than 290 with only 300K steps.
This sharply contrasts with other algorithms, which require 500K steps to achieve convergence while exhibiting inferior performance.
Although the safety cost of S2CD experiences a slight increase during the intermediate stages, it subsequently declines in the later stages and consistently remains lower than that of other algorithms.
Notably, S2CD experiences significantly fewer collisions than other algorithms during the entire training process, with only approximately 2 collisions occurring every 5,000 training steps in the later stage.
This indicates that S2CD demonstrates exceptionally high safety performance during the training process, markedly outperforming other algorithms.

Therefore, among conventional RL algorithms, PPO-based algorithms exhibit excellent performance during model training, while TSF-based algorithms more efficiently acquire superior strategies under teacher guidance.
In particular, the proposed S2CD algorithm exhibited the best overall performance during both the training and testing phases.

\subsection{Model Evaluation}
Upon completing model training, a final performance evaluation was conducted for all models, as illustrated in \autoref{fig:model_evaluate} and presented in \autoref{tab:evaluate_results}.
Notably, CQL consistently maintained the "follow" action during testing without performing any lane-change maneuvers.
Therefore, it is excluded from subsequent analyses.
\begin{figure}[t]
	\centering
	\subcaptionbox{Success Rate\label{fig:success_rate}}{
		\resizebox*{5.2cm}{!}{\includegraphics{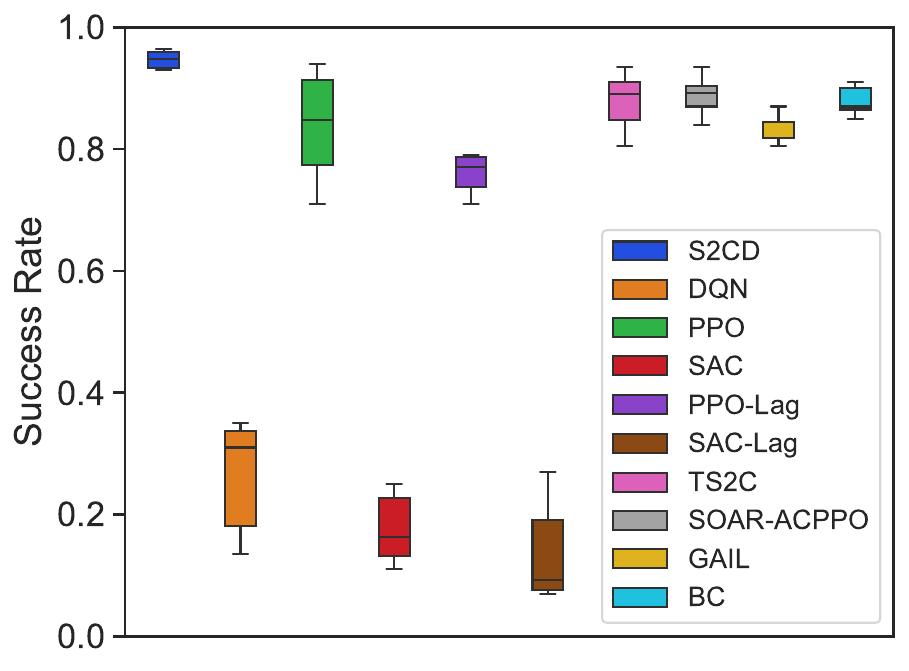}}}
	\subcaptionbox{Return\label{fig:return}}{
		\resizebox*{5.2cm}{!}{\includegraphics{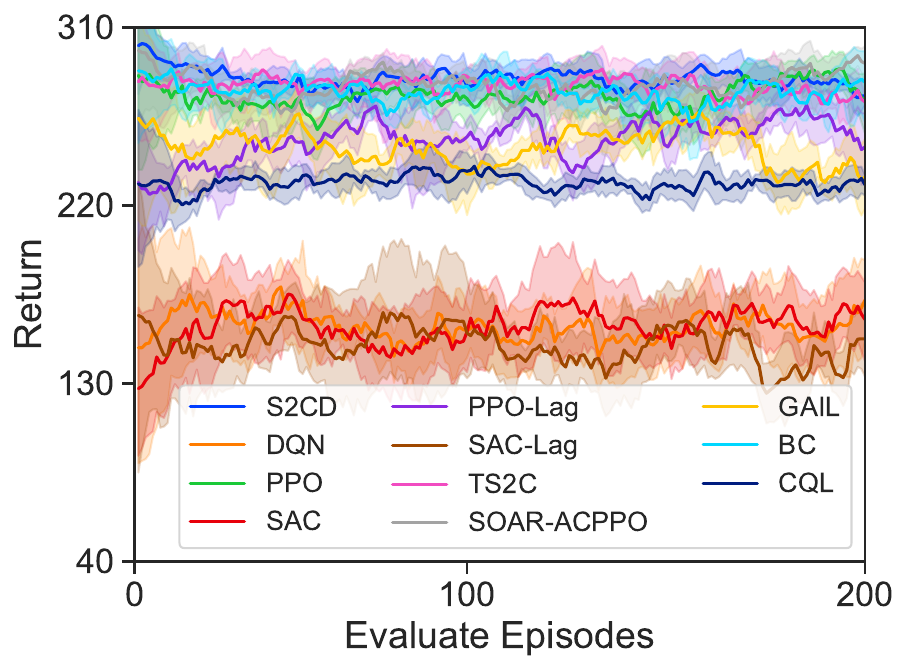}}}
	\subcaptionbox{Safety Cost\label{fig:safety_cost}}{
		\resizebox*{5.2cm}{!}{\includegraphics{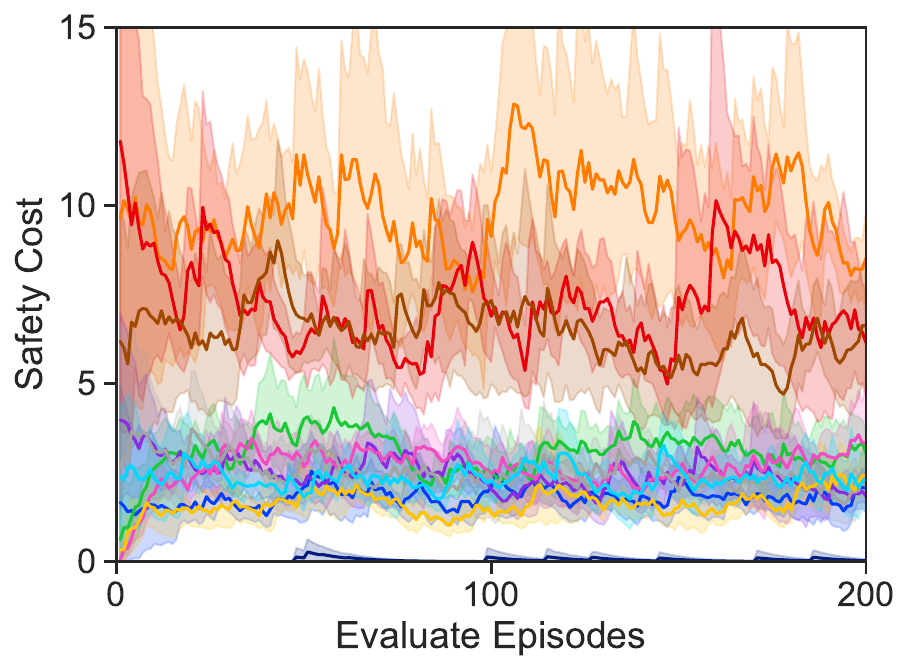}}}
	\subcaptionbox{Efficiency Reward\label{fig:efficiency_reward}}{
		\resizebox*{5.2cm}{!}{\includegraphics{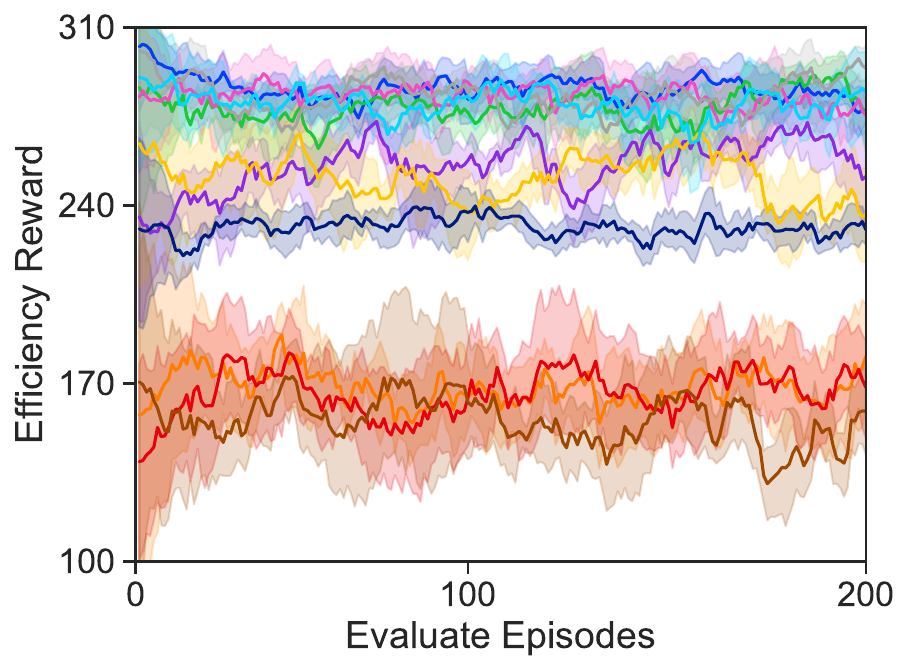}}}
	\subcaptionbox{Speed\label{fig:speed}}{
		\resizebox*{5.2cm}{!}{\includegraphics{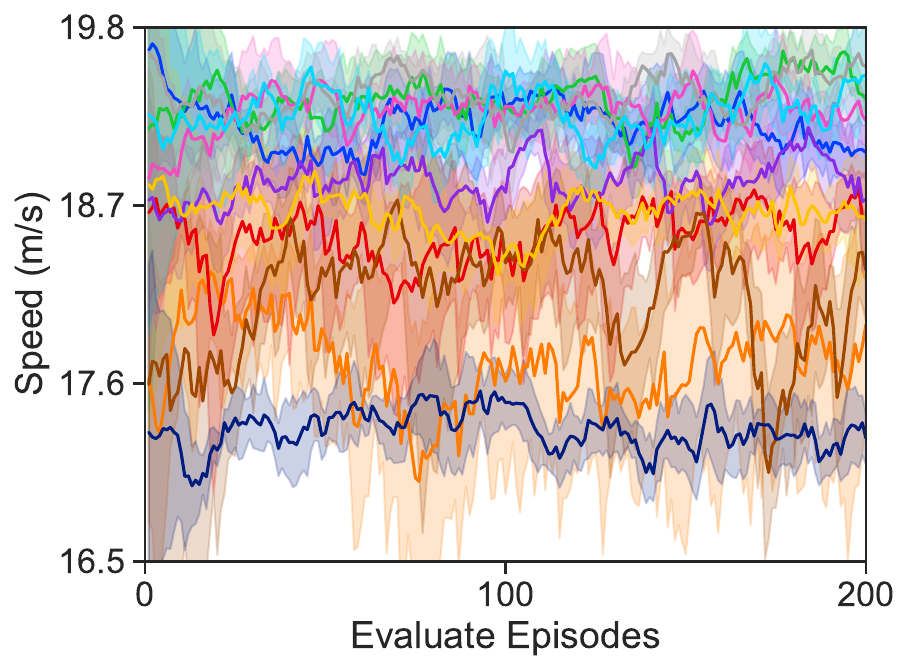}}}
	
	\caption{
		The curves of model evaluation for all algorithms.
		Each model trained with varying random seeds is evaluated twice, and each algorithm undergoes six evaluation runs, with each evaluation consisting of 200 episodes.
	}
	\label{fig:model_evaluate}
\end{figure}

\begin{table}
	\centering
	\caption{The performance of all algorithms in the medium-density traffic scenario}
	\label{tab:evaluate_results}
	{\begin{tabular}{cccccccc} 
			\toprule
			Category & Method & \makecell[c]{Training \\ Steps} & \makecell[c]{Episodic \\ Return} & \makecell[c]{Episodic \\ Reward} & \makecell[c]{Episodic \\ Cost} & \makecell[c]{Episodic \\ Speed (m/s)} & \makecell[c]{Success \\ Rate (\%)} \\
			Value-Based RL & DQN & 500K & 159.85 & 169.60 & 9.75 & 17.74 & 26.50\\
			On-Policy RL & PPO & 500K & 274.82 & 278.00 & 3.18 & \textbf{19.36} & 83.83\\
			Off-Policy RL & SAC & 500K & 160.48 & 167.50 & 7.02 & 18.51 & 17.58\\
			\multirow{2}{*}{Safe RL} & PPO-Lag & 500K & 254.29 & 256.68 & 2.39 & 18.85 & 76.00\\
			& SAC-Lag & 500K & 148.80 & 155.28 & 6.48 & 18.17 & 13.67\\
			\multirow{2}{*}{TSF} & TS2C & 500K & 280.24 & 283.05 & 2.80 & 19.30 & 87.83\\
			& SOAR-ACPPO & 500K & 280.56 & 283.08 & 2.52 & 19.30 & 88.83 \\
			\multirow{2}{*}{IL} & GAIL & 500K & 248.17 & 249.90 & \textbf{1.73} & 18.65 & 83.67\\
			& BC & $/$ & 277.53 & 279.87 & 2.35 & 19.27 & 87.90\\
			Offline RL & CQL & $/$ & 231.47 & 231.51 & $/$ & 17.31 & $/$ \\
			\midrule
			\textbf{Ours} & \textbf{S2CD} & \textbf{300K} & \textbf{283.14} & \textbf{284.96} & 1.82 & 19.19 & \textbf{94.67}\\
			\bottomrule
	\end{tabular}}
\end{table}

The results indicate that PPO-based algorithms significantly outperform both SAC-based algorithms and DQN, especially TSF-based algorithms, consistent with their performance during the training phase.
Notably, S2CD demonstrates superior performance in two key metrics: success rate and return value, achieving a success rate of 94.67\% and a return value of 283.14.
Although S2CD’s safety cost (1.82) and speed (19.19 m/s) are not the highest values, they are very close to the best values of 1.73 and 19.36 m/s, respectively.
This slight discrepancy may be attributed to the fewer training steps undertaken by S2CD.
It is important to note that, for IL algorithms, an expert policy must first be trained using PPO.
This expert policy is subsequently used to collect a substantial amount of high-quality data for training the GAIL and BC models.
Consequently, although GAIL and BC achieve performance levels close to those of S2CD, their training processes are more complex and time-consuming.
In contrast, the PPO algorithm is not influenced by other policies, exploring the environment entirely autonomously.
Despite its relatively low success rate and safety performance, PPO achieves the highest driving speed of 19.36 m/s by prioritizing higher efficiency reward.
Both TS2C and SOAR-ACPPO not only account for the role of teacher guidance but also ensure that the student agent can freely explore the environment in the later stages by reducing the intervention probability.
Consequently, they achieved higher driving speeds (19.30 m/s and 19.30 m/s) and lower safety cost (2.80 and 2.52).
Their return values, reaching 280, are second only to the highest-performing S2CD.
As shown in \autoref{fig:success_rate}, S2CD maintained stable performance due to the effective guidance of the teacher, achieving a success rate variance of only 3.5\% across 6 evaluations, which is significantly lower than that of the PPO algorithm alone.
Moreover, TS2C and SOAR-ACPPO also achieved exceptionally high success rates (87.83\% and 88.83\%), indicating that teacher guidance not only enhances the performance potential of the student models but also improves their safety thresholds.
The performance of the SAC-based and DQN algorithms remained poor, consistent with their training phase results, with success rates not exceeding 30\% and return value below 200, highlighting a significant gap compared to the PPO-based algorithms.

To further validate the generalizability of the S2CD framework, we examined all algorithms in high-density and low-density scenarios.
For high-density scenario, vehicle spacing ranged between 20 and 50 meters, with an average of about 29 vehicles per lane per kilometer.
The evaluation results are shown in \autoref{tab:high-density_results}.
In this scenario, the reduced vehicle spacing increases the complexity of the environment, further degrading the performance of the already underperforming DQN, SAC-Lag and GAIL algorithms.
However, as lane-change opportunities become more limited, the frequency of lane-change by the ego vehicle significantly dropped, resulting in a reduced probability of collisions and, consequently, an improvement in the performance of other algorithms.
S2CD continues to demonstrate the highest success rate (96.83\%) and return value (143.88), while its safety cost (0.77) and speed (15.02 m/s) are also very close to the best values of 0.63 and 15.11 m/s.
For the low-density scenario, vehicle spacing ranged from 90 to 120 meters, with an average of about 9 vehicles per lane per kilometer.
The evaluation results are shown in \autoref{tab:low-density_results}.
In this scenario, vehicles could achieve higher speeds without frequent lane changes, significantly reducing collision probability and thus improving the safety and efficiency of all algorithms.
S2CD achieved the highest success rate (98.75\%) and return value (341.01), while also demonstrating the lowest safety cost (0.51).
Additionally, its driving speed (21.32 m/s) is close to the highest speed achieved by the PPO algorithm, which is 21.68 m/s.
Similarly, TS2C and SOAR-ACPPO, which are based on the TSF framework, ranked just behind S2CD in overall performance, outperforming other algorithms.
These results are consistent with those in the moderate-density scenario, further validating the effectiveness of teacher-guided learning.

\begin{table}[t]
	\centering
	\caption{The performance of all algorithms in the high-density traffic scenario}
	\label{tab:high-density_results}
	{\begin{tabular}{cccccc} 
			\toprule
			Method & \makecell[c]{Episodic \\ Return} & \makecell[c]{Episodic \\ Reward} & \makecell[c]{Episodic \\ Cost} & \makecell[c]{Episodic \\ Speed (m/s)} & \makecell[c]{Success \\ Rate (\%)} \\
			\midrule
			DQN & 46.48 & 66.32 & 19.84 & 14.12 & 18.83\\
			PPO & 141.02 & 144.44 & 3.42 & \textbf{15.11} & 92.83\\
			SAC & 58.28 & 76.93 & 18.65 & 14.99 & 18.00\\
			PPO-Lag & 127.93 & 134.56 & 6.63 & 15.10 & 79.17\\
			SAC-Lag & 62.78 & 71.55 & 8.77 & 14.83 & 13.50\\
			TS2C & 141.30 & 142.04 & 0.75 & 15.06 & 93.00\\
			SOAR-ACPPO & 142.89 & \textbf{146.37} & 3.47 & 15.11 & 94.25\\
			GAIL & 120.23 & 123.86 & 3.63 & 15.07 & 75.00\\
			BC & 141.70 & 142.32 & \textbf{0.63} & 14.99 & 93.75\\
			CQL & 140.12 & 140.12 & $/$ & 14.88 & $/$ \\			
			\midrule
			\textbf{S2CD} & \textbf{143.88} & 144.66 & 0.77 & 15.02 & \textbf{96.83}\\
			\bottomrule
	\end{tabular}}
\end{table}

\begin{table}
	\centering
	\caption{The performance of all algorithms in the low-density traffic scenario}
	\label{tab:low-density_results}
	{\begin{tabular}{cccccc} 
			\toprule
			Method & \makecell[c]{Episodic \\ Return} & \makecell[c]{Episodic \\ Reward} & \makecell[c]{Episodic \\ Cost} & \makecell[c]{Episodic \\ Speed (m/s)} & \makecell[c]{Success \\ Rate (\%)} \\
			\midrule
			DQN & 274.54 & 278.86 & 4.32 & 20.55 & 57.67\\
			PPO & 339.63 & 341.11 & 1.48 & \textbf{21.68} & 90.00\\
			SAC & 274.75 & 278.91 & 4.15 & 20.71 & 52.50\\
			PPO-Lag & 328.58 & 330.52 & 1.94 & 21.18 & 90.50\\
			SAC-Lag & 261.22 & 267.44 & 6.21 & 20.75 & 44.33\\
			TS2C & 336.94 & 338.03 & 1.09 & 21.51 & 91.92\\
			SOAR-ACPPO & 339.95 & 340.82 & 0.87 & 21.52 & 95.00\\
			GAIL & 320.32 & 321.41 & 1.09 & 20.99 & 89.50\\
			BC & 340.60 & \textbf{341.57} & 0.97 & 21.46 & 91.67\\
			CQL & 274.77 & 274.80 & $/$ & 18.78 & $/$ \\
			\midrule
			\textbf{S2CD} & \textbf{341.01} & 341.52 & \textbf{0.51} & 21.32 & \textbf{98.75}\\
			\bottomrule
	\end{tabular}}
\end{table}

The preceding discussion emphasizes that our proposed S2CD framework effectively balances the benefits of teacher guidance with autonomous exploration, thereby establishing a driving strategy that integrates both training efficiency and safety.
Importantly, S2CD demonstrates the most superior performance among all the compared algorithms.

\subsection{Ablation Study}
To assess the contributions of every part within the S2CD framework and validate its effectiveness, we conducted ablation experiments.
By gradually removing or replacing specific modules within the framework, we observed variations in performance.
The ablation study validated the framework’s design rationale and demonstrated the role that each module plays in enhancing performance.

\subsubsection{The Impact of Different Teacher Quality}
To explore the impact of guidance from teachers of varying quality on student performance, we trained a high-quality teacher model and a low-quality teacher model in a simple simulation environment (Highway-Env), as well as an additional teacher model in a complex simulation environment (Carla).
These 3 teacher models were subsequently applied to the S2CD framework to guide the learning of the student agent.

\autoref{tab:calculate_costs} presents the computational costs associated with different teacher models, including training steps, number of processes, time spent, GPU memory footprint, memory footprint, and CPU utilization.
In the Highway-Env environment, training a low-quality teacher model requires 200K steps, whereas achieving a high-quality teacher model can be accomplished by simply doubling the number of data collection processes.
Although the training steps and CPU utilization approximately doubled, the increases in time spent, GPU memory footprint, and memory footprint were minimal, with the training time remaining almost unchanged.
Therefore, in the Highway-Env environment, a high-quality teacher model can be obtained at minimal computational cost.
In contrast, training for the same 400K steps in the Carla simulation environment required nearly 5 times the time spent compared to Highway-Env, with the GPU memory footprint, memory footprint, and CPU utilization more than doubling.
This indicates that the cost of training teacher models in complex simulation environments is significantly greater than in simpler ones.
The three trained teacher models were integrated into the S2CD framework to guide the learning of the student in Carla, with the training and evaluation results presented in \autoref{fig:dif_teacher_results} and \autoref{tab:dif_teacher_results}.

\begin{table}
	\centering
	\caption{Computational cost for different teacher models}
	\label{tab:calculate_costs}
	{\begin{tabular}{ccccccc} 
			\toprule
			\makecell[c]{Computational Cost} & \makecell[c]{Training \\ Steps} & \makecell[c]{Training \\ Processes} & \makecell[c]{Time \\ Spent (h)} & \makecell[c]{GPU Memory \\ Footprint (GB)} & \makecell[c]{Memory \\ Footprint (GB)} & \makecell[c]{CPU \\ Utilization (\%)} \\
			\midrule			
			Simple Teacher-High & 400K & 4 & 1.98 & 4.68 & 3.82 & 20.38\\
			Simple Teacher-Low & 200K & 2 & 1.96 & 3.03 & 3.06 & 12.60\\
			Complex Teacher & 400K & 2 & 9.58 & 10.60 & 8.74 & 67.18\\
			\bottomrule
	\end{tabular}}
\end{table}

\begin{figure}[t]
	\centering
	\subcaptionbox{Test Return\label{fig:dif_test_return}}{
		\resizebox*{5.2cm}{!}{\includegraphics{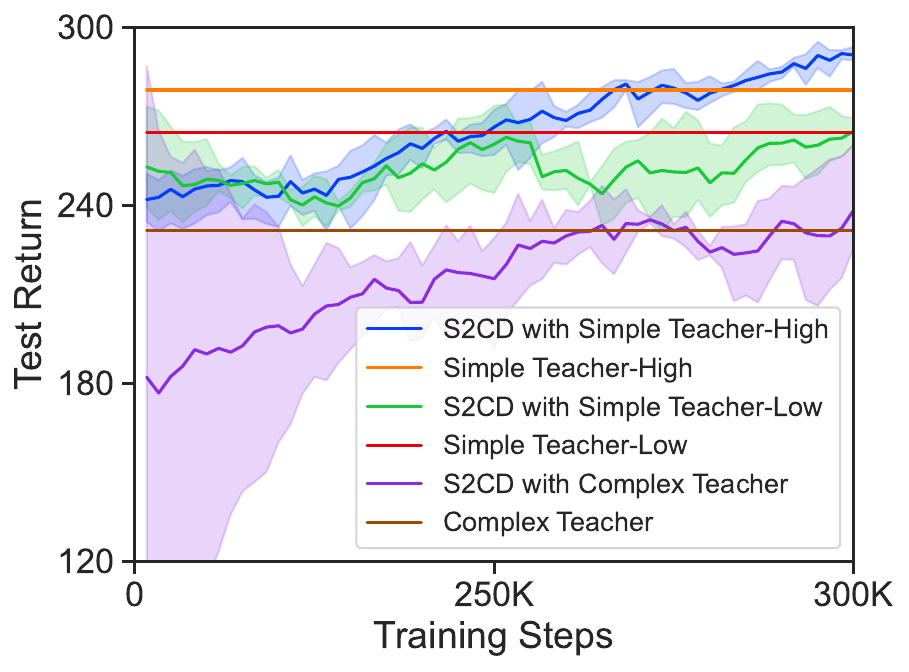}}}\hspace{1cm}
	\subcaptionbox{Test Safety Cost\label{fig:dif_test_safety_cost}}{
		\resizebox*{5.2cm}{!}{\includegraphics{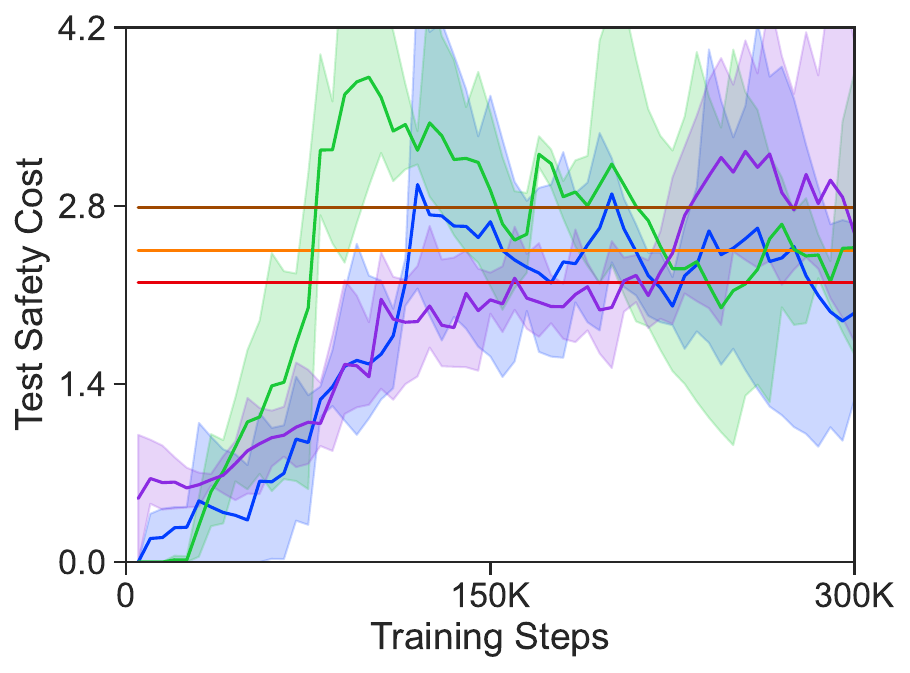}}}
	\subcaptionbox{Training Collision Counts\label{fig:dif_training_collision_counts}}{
		\resizebox*{5.2cm}{!}{\includegraphics{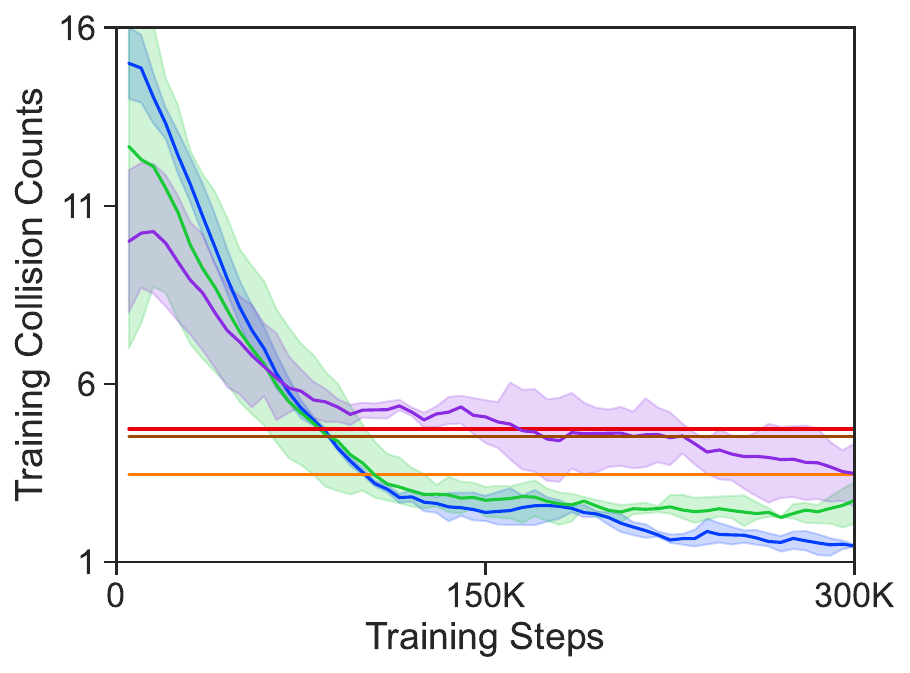}}}\hspace{1cm}
	\subcaptionbox{Success Rate\label{fig:dif_success_rate}}{
		\resizebox*{5.2cm}{!}{\includegraphics{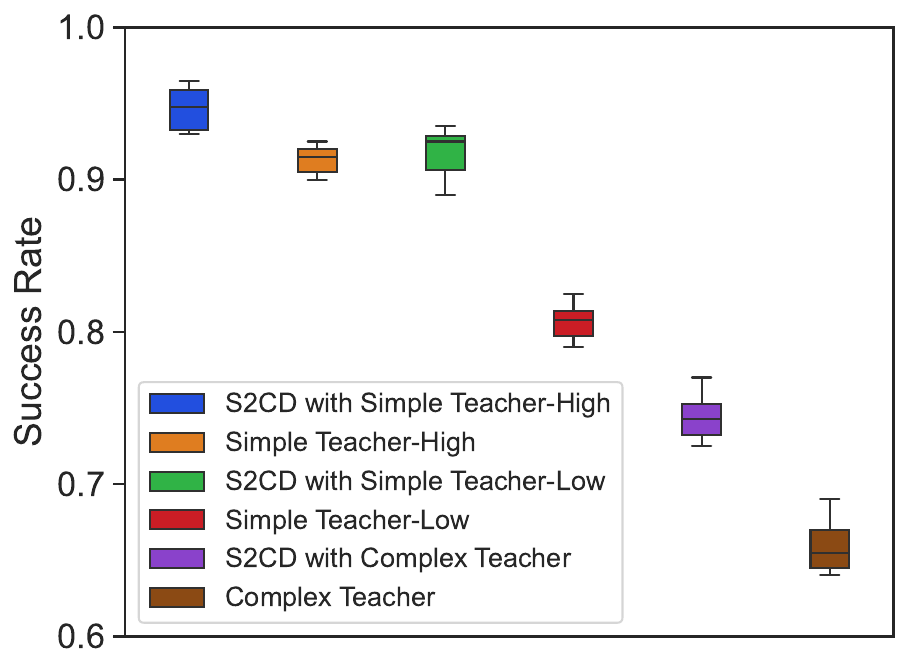}}}
	
	\caption{
		The training and evaluation processes involve different teacher models:
		1. Model Training: Every algorithm is trained with three random seeds, accumulating 300K training steps per case.
		Two evaluation episodes are conducted every 5,000 training steps, and their average value is recorded as the result.
		Throughout model training, we monitor the return value, safety cost, and collision counts.
		2. Model Evaluation: Every model trained with different random seeds is evaluated twice.
		In total, every algorithm undergoes six evaluation runs, with each run consisting of 200 episodes.
		During model evaluation, the success rate is recorded.
	}
	\label{fig:dif_teacher_results}
\end{figure}

\begin{table}
	\centering
	\caption{The performance of S2CD under the guidance of teachers at different levels}
	\label{tab:dif_teacher_results}
	{\begin{tabular}{cccccc} 
			\toprule
			Experiment & \makecell[c]{Episodic \\ Return} & \makecell[c]{Episodic \\ Reward} & \makecell[c]{Episodic \\ Cost} & \makecell[c]{Episodic \\ Speed (m/s)} & \makecell[c]{Success \\ Rate (\%)} \\
			\midrule			
			S2CD with Simple Teacher-High & \textbf{283.14} & \textbf{284.96} & \textbf{1.82} & 19.19 & \textbf{94.67}\\
			Simple Teacher-High & 281.21 & 283.69 & 2.48 & \textbf{19.28} & 91.30\\
			S2CD with Simple Teacher-Low & 272.19 & 274.32 & 2.13 & 18.98 & 91.75\\
			Simple Teacher-Low & 269.11 & 271.93 & 2.82 & 19.20 & 80.67\\
			S2CD with Complex Teacher & 243.59 & 245.57 & 1.98 & 18.87 & 74.42\\
			Complex Teacher & 238.64 & 241.03 & 2.39 & 19.07 & 66.00\\
			\bottomrule
	\end{tabular}}
\end{table}

The results demonstrate that higher-quality teacher models enable the student to acquire more accurate knowledge, resulting in improved performance during the later stages of training and in the final evaluation.
Specifically, the return, the probability of avoiding collisions and the task success rate during training were higher with the high-quality teacher compared to the low-quality teacher.
This indicates that high-performance teachers can significantly enhance the safety and performance of the student's training.
Additionally, we implemented the weaning mechanism to gradually decrease the teacher's influence on the student agent in the later stages, allowing the student to explore the environment more independently.
This mechanism elevated the algorithm's performance ceiling, enabling the student to ultimately surpass the teacher in key metrics such as return and success rates.
Notably, although the teacher model trained in the Carla environment consumed significant computational resources, its performance was inferior to that of the low-quality teacher trained in the Highway-Env environment.
The return and success rates were only 238.64 and 66.00\%, respectively, significantly lower than those of the teachers trained in simpler environments.
Furthermore, when guided by this teacher model, the student agent's final performance was lower than that of the simple teacher.
The return was approximately 40 lower, and the key metric, success rate, was about 20\% lower.
This demonstrates that training teacher models in a simple simulation environment not only saves substantial computational costs but also yields better performance, providing more effective guidance to the student agent.

\subsubsection{Ablations of Guardian Mechanism}
In this work, we introduce substantial enhancements to the traditional PPO algorithm.
First, we utilized data from both the teacher and student models concurrently, effectively doubling the available training data and significantly improving sample efficiency.
We then adapted the clipping factor of the algorithm based on the varying importance of these two data types, further enhancing the algorithm's learning efficiency.
Additionally, to ensure that the student policy rapidly approximates the teacher policy in the early stages of training, we employed the KL divergence between the two policies as a constraint of policy update.
We subsequently applied the Lagrangian method to transform the optimization problem into an unconstrained format for policy updates.
Finally, we implemented an annealing mechanism to preserve the agent’s exploratory capabilities, ensuring it can overcome the limitations of the teacher's performance.
By incorporating the 4 modules described above, our proposed algorithm not only significantly enhances learning efficiency but also improves the model's performance and safety.
To assess the contribution of each module to the overall performance, we conducted ablation experiments.
In these experiments, we sequentially removed each module and analyzed the resulting performance changes to verify the necessity of each component.
The training and evaluation results from the ablation experiments are presented in \autoref{fig:ablations_results} and \autoref{tab:ablations_results}.

\begin{figure}[t]
	\centering
	\subcaptionbox{Test Return\label{fig:ablation_test_return}}{
		\resizebox*{5.2cm}{!}{\includegraphics{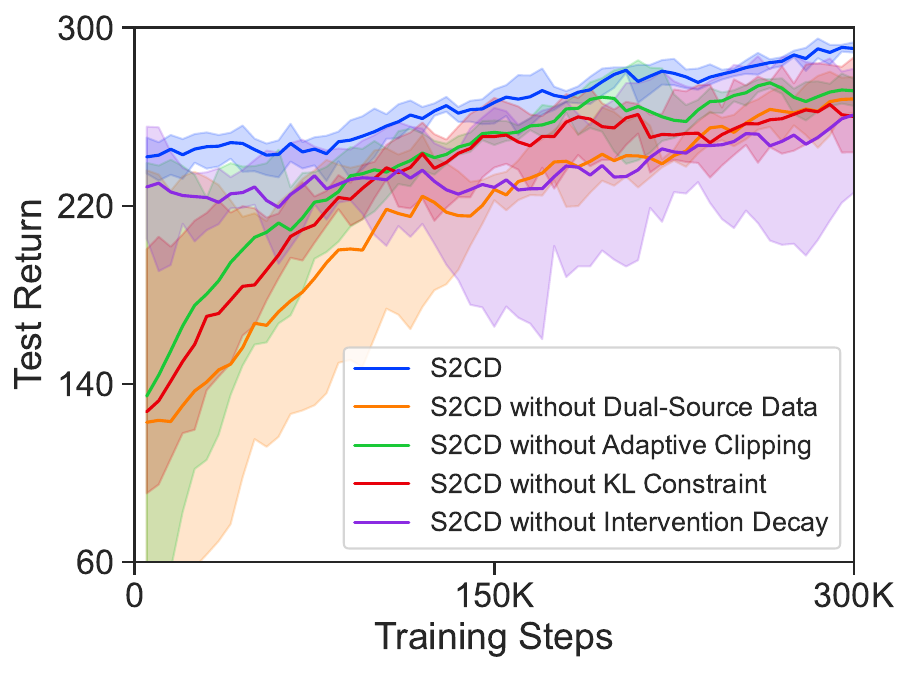}}}
	\subcaptionbox{Test Cost\label{fig:ablation_test_safety_cost}}{
		\resizebox*{5.2cm}{!}{\includegraphics{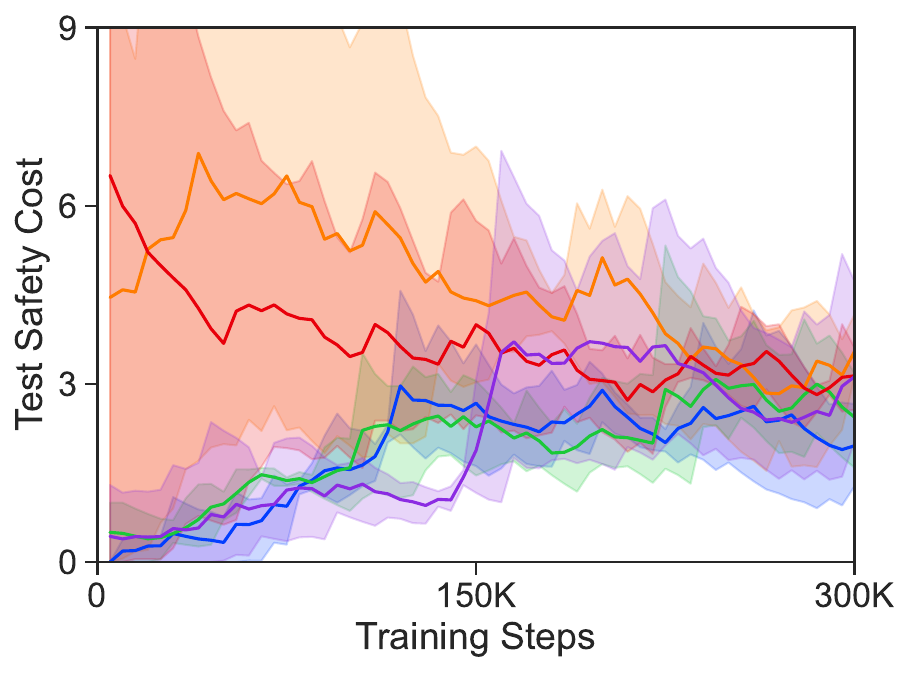}}}
	\subcaptionbox{Test Speed\label{fig:ablation_test_speed}}{
		\resizebox*{5.2cm}{!}{\includegraphics{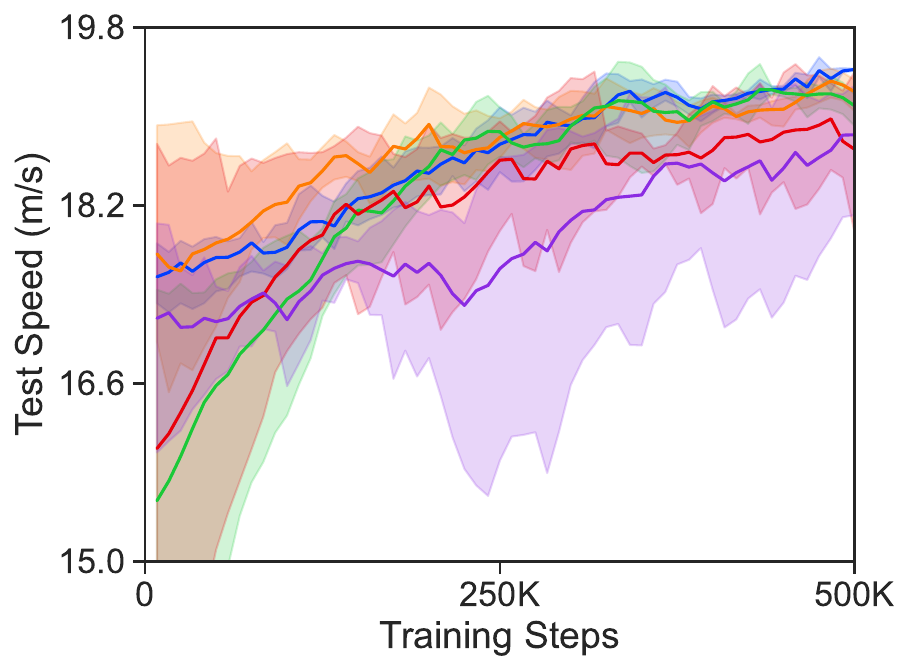}}}
	\subcaptionbox{Training Collision Counts\label{fig:ablation_training_collision_counts}}{
		\resizebox*{5.2cm}{!}{\includegraphics{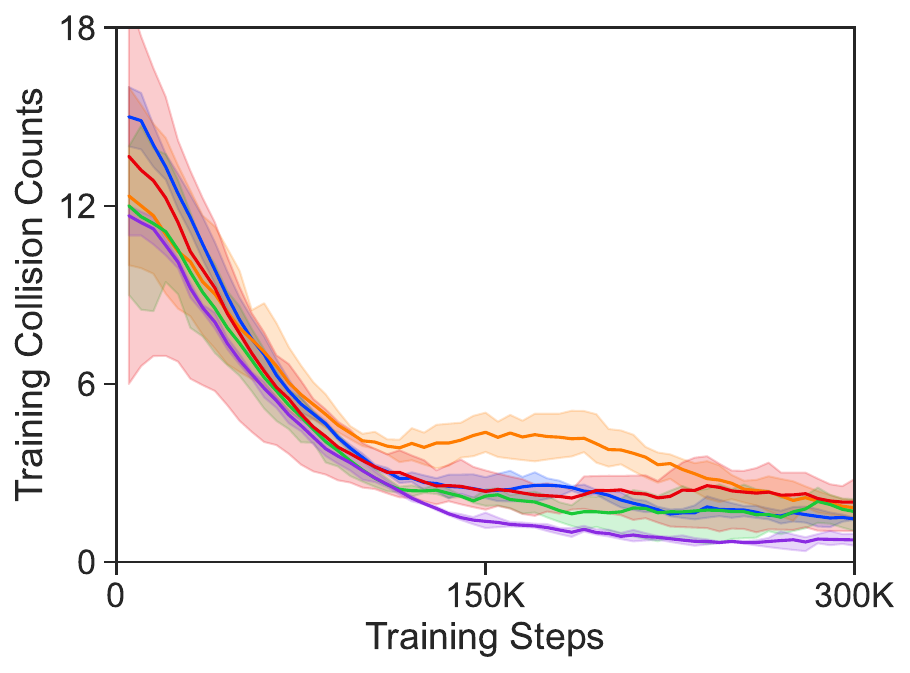}}}
	\subcaptionbox{Success Rate\label{fig:ablation_success_rate}}{
		\resizebox*{5.2cm}{!}{\includegraphics{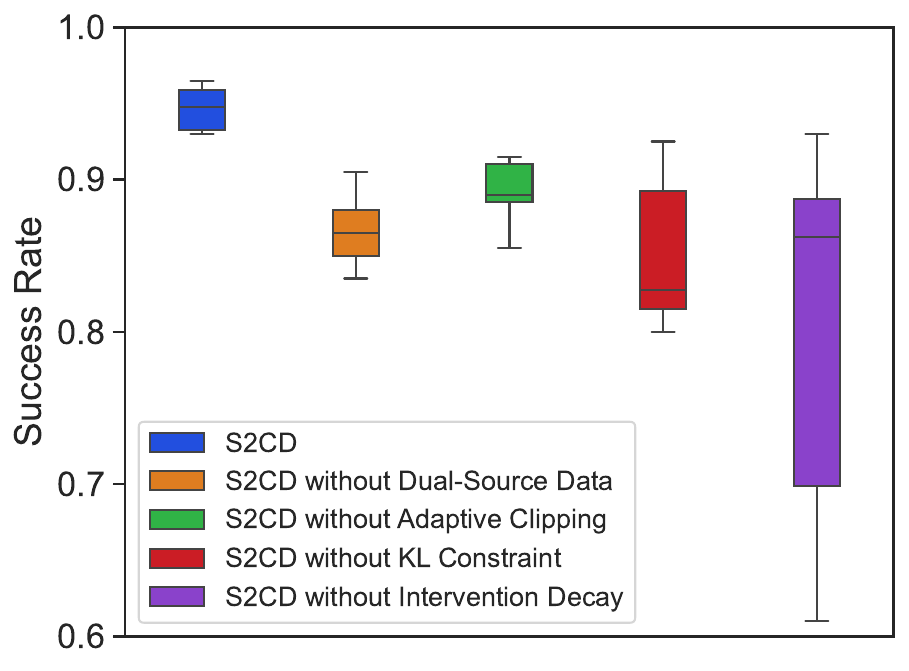}}}
	\caption{
		The training and evaluation processes of the ablation experiment. 
		1. Model Training:
		Each algorithm is trained with three different random seeds, accumulating a total of 300K training steps for each case. 
		Two evaluation episodes are conducted for every 5,000 training steps, and the average value of these two episodes is taken as the result. 
		During model training, we record the return value, safety cost, speed, and collision counts.
		2. Model Evaluation:
		Each model trained with varying random seeds is evaluated twice, and each algorithm undergoes six evaluation runs, with each evaluation consisting of 200 episodes. 
		For model evaluation, we record the success rate.
	}
	\label{fig:ablations_results}
\end{figure}

\begin{table}
	\centering
	\caption{The performance of S2CD in ablation experiments}
	\label{tab:ablations_results}
	{\begin{tabular}{cccccc} 
			\toprule
			Experiment & \makecell[c]{Episodic \\ Return} & \makecell[c]{Episodic \\ Reward} & \makecell[c]{Episodic \\ Cost} & \makecell[c]{Episodic \\ Speed (m/s)} & \makecell[c]{Success Rate \\ (\%)} \\
			\midrule			
			S2CD & \textbf{283.14} & \textbf{284.96} & \textbf{1.82} & 19.19 & \textbf{94.67}\\
			S2CD without Dual-Source Data & 277.28 & 280.45 & 3.16 & \textbf{19.32} & 86.70\\
			S2CD without Adaptive Clipping & 274.24 & 276.68 & 2.45 & 19.08 & 89.10\\
			S2CD without KL Constraint & 262.09 & 265.94 & 3.85 & 18.63 & 85.08\\
			S2CD without Intervention Decay & 255.27 & 259.00 & 3.73 & 18.74 & 80.08\\
			\bottomrule
	\end{tabular}}
\end{table}

The experimental results indicate that the performance of the S2CD framework declines to some extent with the removal of any module.
Specifically, when the teacher model's intervention does not gradually diminish, the student performs well in the early stages of training and maintains a low collision rate throughout.
However, in the later stages, the student struggles to break free from the influence of the teacher's strategy.
This not only leads to the student learning many of the teacher's incorrect instructions but also hinders the student's ability to explore freely and acquire more accurate knowledge.
Ultimately, the return and success rates were only 255.27 and 80.08\%, respectively, with a relatively low average speed of 18.74 m/s.
This indicates that the gradual reduction of intervention is crucial for ensuring that the student agent ultimately surpasses the performance of the teacher model.
As shown in \autoref{fig:ablation_test_return}, the removal of the "Dual-Source Data", "Adaptive Clipping", and "KL Constraint" modules significantly reduces the framework's performance in the early stages of training and negatively impacts the overall performance, leading to lower final outcomes for the model.
The absence of these three modules diminishes the model's learning efficiency, thereby limiting the knowledge gained by the student agent.
It is important to note that due to the imperfections in the teacher model's Return network and Q-Value network components, the use of additional training data and the inclusion of KL constraints carry the risk of the student learning incorrect information.
This risk must be mitigated by gradually reducing the teacher's intervention.

The adaptive clipping module primarily accelerates the optimization process.
Thus, after its removal, the model's final performance remains closest to that of the complete framework, with its success rate being only 5.57\% lower than that of the full framework.
The analysis above indicates that the "Dual-Source Data", "Adaptive Clipping", and "KL Constraint" modules primarily enhance the early-stage learning efficiency of the model, enabling the student to quickly absorb the knowledge provided by the teacher.
In contrast, the intervention reduction module improves the performance ceiling of the student model by limiting excessive teacher intervention, aligning with the design goals of the S2CD framework.

\section{Conclusion}
\label{sec:conclusion}

This paper introduces a novel framework, S2CD, based on knowledge transfer techniques.
The S2CD framework first trains a teacher model in a simplified simulation environment and then utilizes this model to guide the student agent in a more complex environment, enhancing the safety and efficiency of the training process.
The paper also presents an innovative RL algorithm, ACPPO+, which leverages samples generated by both teacher and student policies while dynamically adjusting the clipping factor based on sample importance, thereby improving learning efficiency.
Moreover, the KL divergence between the teacher’s and student’s policies is incorporated as a constraint in model updates, solved using the Lagrangian method to enable the student agent to quickly adopt the policy of teacher.
To enhance the student agent’s performance, a gradual weaning strategy reduces teacher intervention, ensuring that the student becomes more autonomous and independently explores optimal policies.
Experimental results in highway lane-change scenarios demonstrate that compared to traditional RL, IL, Safe RL, TSF, and Offline RL, the S2CD framework significantly enhances both learning efficiency and model performance, while also reducing training costs.
Most importantly, even with suboptimal teacher performance, the S2CD framework enhances safety during training.
The framework is also applicable beyond varying simulation environments and offers a foundation for future knowledge transfer research between simulation and real-world environments.

The primary limitation of this study is that all experiments were conducted in simulated environments, without utilizing teacher models trained in simulation environments to guide student models in real-world conditions.
Therefore, the next step in our research will involve real-world vehicle experiments to further validate the effectiveness and practicality of the proposed framework.
Additionally, this study has not addressed mixed traffic scenarios involving intelligent and traditional vehicles, which represents an important direction for future research to expand the applicability and coverage of the S2CD framework.

\bibliographystyle{elsarticle-harv} 
\bibliography{ref}

\appendix

\section{Proof of Theorem 3}
\label{Appendix:Theorem_3}
\begin{theorem}[Restatement of Theorem 3]
	With the switch function, the return of the mixed behavior policy $J(\theta_\text{{mix}}) $ is lower and upper bounded by:
	\begin{equation}
		\label{eq:effect_bound_1}
		J(\theta_{t}) + \frac{\sqrt{2}(1 - \omega)R_{\max}}{(1 - \gamma)^2} \sqrt{H - \kappa}
		\geq J(\theta_\text{{mix}}) 
		\geq J(\theta_{t}) - \frac{\sqrt{2}(1 - \omega)R_{\max}}{(1 - \gamma)^2} \sqrt{H - \kappa}
	\end{equation}
\end{theorem}

\begin{proof}
	Using the mixed behavior policy as defined in \autoref{eq:mix_policy}, the difference between $J(\theta_{\text{mix}})$ and the return of the teacher's policy $J(\theta_{t})$, is expressed as follows:
	\begin{equation}
		\begin{split}
			\left| J(\theta_\text{{mix}}) - J(\theta_{t}) \right| 
			&\leq \frac{R_{\max}}{(1 - \gamma)^2} \mathbb{E}_{s \sim d_{\text{mix}}} \left\| \pi^{\text{mix}}(\cdot | s) - \pi^{t}(\cdot | s) \right\|_1 \\
			&= \frac{R_{\max}}{(1 - \gamma)^2} \mathbb{E}_{s \sim d_{\text{mix}}} \left\| \mathcal{T}(s)\pi^{t}(\cdot | s) + (1 - \mathcal{T}(s))\pi^{s}(\cdot | s) - \pi^{t}(\cdot | s) \right\|_1 \\
			&= \frac{(1 - \omega)R_{\max}}{(1 - \gamma)^2} \mathbb{E}_{s \sim d_{\text{mix}}} \left\| \pi^{s}(\cdot | s) - \pi^{t}(\cdot | s) \right\|_1 \\
			&\leq \frac{\sqrt{2}(1 - \omega)R_{\max}}{(1 - \gamma)^2} \mathbb{E}_{s \sim d_{\text{mix}}} \sqrt{D_{\text{KL}}(\pi^{t}(\cdot | s) \| \pi^{s}(\cdot | s))} \\
			&= \frac{\sqrt{2}(1 - \omega)R_{\max}}{(1 - \gamma)^2} \mathbb{E}_{s \sim d_{\text{mix}}} \sqrt{\mathbb{E}_{a \sim \pi^{t}(\cdot | s)} \left[ \log \frac{\pi^{t}(a|s)}{\pi^{s}(a|s)} \right]} \\
			&= \frac{\sqrt{2}(1 - \omega)R_{\max}}{(1 - \gamma)^2} \mathbb{E}_{s \sim d_{\text{mix}}} \sqrt{\mathcal{H}(\pi^{t}(\cdot | s)) - \kappa} \\
			&\leq \frac{\sqrt{2}(1 - \omega)R_{\max}}{(1 - \gamma)^2} \sqrt{H - \kappa}.
		\end{split}
	\end{equation}
	
	So we can obtain:
	\begin{equation}
		\label{eq:effect_bound_proof}
		\begin{aligned}
			&\frac{\sqrt{2}(1 - \omega)R_{\max}}{(1 - \gamma)^2} \sqrt{H - \kappa} 
			\geq J(\theta_\text{{mix}}) - J(\theta_{t})
			\geq -\frac{\sqrt{2}(1 - \omega)R_{\max}}{(1 - \gamma)^2} \sqrt{H - \kappa} \\
			&J(\theta_{t}) + \frac{\sqrt{2}(1 - \omega)R_{\max}}{(1 - \gamma)^2} \sqrt{H - \kappa}
			\geq J(\theta_\text{{mix}}) 
			\geq J(\theta_{t}) - \frac{\sqrt{2}(1 - \omega)R_{\max}}{(1 - \gamma)^2} \sqrt{H - \kappa}
		\end{aligned}
	\end{equation}
	
	Proof completed.
	
\end{proof}

\section{Proof of Theorem 4}
\label{Appendix:Theorem_4}
\begin{theorem}[Restatement of Theorem 4]
	The expected cumulative reward obtained by learning from $\pi^{\text{mix}}$ is guaranteed to be greater than or equal to the expected cumulative reward obtained by the $\pi^{t}$:
	\begin{equation}
		\mathbb{E}_{\pi^{\text{mix}}} \left[ \sum_{t=0}^{H} \gamma^t r(s_t, a_t) \right] 
		\geq \mathbb{E}_{\pi^{t}} \left[ \sum_{t=0}^{H} \gamma^t r(s_t, a_t) \right]
	\end{equation}
\end{theorem}

\begin{proof}
	We will prove this using the method of induction. When $H=0$, according to the switch function, the mixed policy $\pi^{\text{mix}}$ will choose the action that results in the highest reward:
	\begin{equation}
		\mathbb{E}_{\pi^{\text{mix}}} \left[ r(s, a) \right] = 
		\max \left( \mathbb{E}_{\pi^{t}} \left[ r(s, a) \right], \mathbb{E}_{\pi^{s}} \left[ r(s, a) \right] \right)
	\end{equation}
	
	Now, assume that the theorem holds for some horizon $k \geq 0$:
	
	\begin{equation}
		\mathbb{E}_{\pi^{\text{mix}}} \left[ \sum_{t=0}^{k} \gamma^t r(s_t, a_t) \right] 
		= \max \left( \mathbb{E}_{\pi^{t}} \left[ \sum_{t=0}^{k} \gamma^t r(s_t, a_t) \right], \mathbb{E}_{\pi^{s}} \left[ \sum_{t=0}^{k} \gamma^t r(s_t, a_t) \right] \right)
	\end{equation}
	
	With this assumption in place, we need to demonstrate that the theorem holds for the horizon $k+1$. 
	To achieve this, define the value function for the policy $\pi$ as:
	\begin{equation}
		V_\pi(s) = \mathbb{E}_{\pi} \left[ \sum_{t=0}^{k+1} \gamma^t r(s_t, a_t) \ \middle| \ s_0 = s \right]
		= r(s, \pi(s)) + \gamma \mathbb{E}_{s' \sim P(s'|s, \pi(s))} \left[ V_\pi(s') \right]
	\end{equation}
	
	For $\pi^{\text{mix}}$, this equation becomes:
	\begin{equation}
		V_{\pi^{\text{mix}}}(s) = \max \left( r(s, \pi^{t}(s)) + \gamma \mathbb{E}_{s'} \left[ V_{\pi^{\text{mix}}}(s') \right], r(s, \pi^{s}(s)) + \gamma \mathbb{E}_{s'} \left[ V_{\pi^{\text{mix}}}(s') \right] \right)
	\end{equation}
	
	By the inductive hypothesis, we know that $V_{\pi^{\text{mix}}}(s') = \max \left( V_{\pi^{t}}(s'), V_{\pi^{s}}(s') \right)$ for all $s'$.
	Substituting this into the above equation:
	\begin{equation}
		\begin{split}
			V_{\pi^{\text{mix}}}(s) &= \max \left( r(s, \pi^{t}(s)) + \gamma \mathbb{E}_{s'} \left[ \max \left( V_{\pi^{t}}(s'), V_{\pi^{s}}(s') \right) \right], r(s, \pi^{s}(s)) + \gamma \mathbb{E}_{s'} \left[\max \left( V_{\pi^{t}}(s'), V_{\pi^{s}}(s') \right) \right] \right) \\
			&\geq \max \left( r(s, \pi^{t}(s)) + \gamma \mathbb{E}_{s'} \left[ \left( V_{\pi^{t}}(s') \right) \right], r(s, \pi^{s}(s)) + \gamma \mathbb{E}_{s'} \left[\left(V_{\pi^{s}}(s') \right) \right] \right) \\
			&= \max \left(V_{\pi^{t}}(s) , V_{\pi^{s}}(s) \right) \\
			&\geq V_{\pi^{t}}(s)
		\end{split}		
	\end{equation}
	
	This result demonstrates that the theorem holds for horizon $k + 1$.
	By the principle of mathematical induction, we can conclude that the theorem holds for all finite horizons $H \geq 0$.
	Thus, the hybrid policy $\pi^{\text{mix}}$ outperforms the existing physics-based policy $\pi^{t}$, namely:
	\begin{equation}
		\mathbb{E}_{\pi^{\text{mix}}} \left[ \sum_{t=0}^{H} \gamma^t r(s_t, a_t) \right] 
		\geq \mathbb{E}_{\pi^{t}} \left[ \sum_{t=0}^{H} \gamma^t r(s_t, a_t) \right]
	\end{equation}
	
	Proof completed.
	
\end{proof}

\section{Hyper-parameters}
\label{Appendix:Hyper-parameters}

Hyper-parameters of different baseline algorithms.
\begin{table}[ht]
	\centering
	\begin{minipage}{0.45\textwidth}
		\caption{PPO/PPO-Lag}
		\label{tab:PPO/PPO-Lag}
		\centering
		\begin{tabular}{cc}
			\toprule
			\textbf{Hyper-parameters} & \textbf{Value} \\
			\midrule
			Learning rate & 0.0005 \\
			Learning rate decay & True \\
			Total steps per episode & 5,000 \\
			Total training steps & 500K \\
			Optimizer & Adam W \\
			Mini batch size & 64 \\
			Discount factor $\gamma$ & 0.96 \\
			Lambda entropy $\beta$ & 0.01 \\
			Clip parameter $\epsilon$ & 0.2 \\
			Lambda advantage $\lambda$ & 0.98 \\
			\midrule
			Penalty parameter & 0.2 \\
			Cost limit for PPO-Lag & 0 \\
			\bottomrule
		\end{tabular}
	\end{minipage}
	\hspace{0.05\textwidth} 
	\begin{minipage}{0.45\textwidth}
		\caption{SAC/SAC-Lag/CQL}
		\label{tab:SAC/SAC-Lag/CQL}
		\centering
		\begin{tabular}{cc}
			\toprule
			\textbf{Hyper-parameters} & \textbf{Value} \\
			\midrule
			Learning rate & 0.0005 \\
			Learning rate decay & True \\
			Total steps per episode & 5,000 \\
			Total training steps & 500K \\
			Optimizer & Adam W \\
			Mini batch size & 64 \\
			Discount factor $\gamma$ & 0.96 \\
			Buffer size & 1e6 \\
			Temperature parameter $\nu$ & -2 \\
			Learning rate of $\alpha$ & 0.001 \\
			Target entropy & -1 \\
			\midrule
			Penalty parameter & 0.2 \\
			Cost limit for SAC-Lag & 0 \\
			\midrule
			Dataset Size & 1e6 \\
			CQL loss temperature & 3 \\
			Min Q weight multiplier & 0.2 \\
			\bottomrule
		\end{tabular}
	\end{minipage}
\end{table}

\begin{table}[ht]
	\centering
	\begin{minipage}{0.45\textwidth}
		\caption{GAIL/BC}
		\label{tab:GAIL}
		\centering
		\begin{tabular}{cc}
			\toprule
			\textbf{Hyper-parameters} & \textbf{Value} \\
			\midrule
			Learning rate & 0.0005 \\
			Learning rate decay & True \\
			Total steps per episode & 5,000 \\
			Total training steps & 500K \\
			Optimizer & Adam W \\
			Mini batch size & 64 \\
			Discount factor $\gamma$ & 0.96 \\
			Dataset Size & 40 K \\
			Sample batch size & 5,000 \\
			Discriminator learning rate & 0.001 \\
			\bottomrule
		\end{tabular}
	\end{minipage}
	\hspace{0.05\textwidth} 
	\begin{minipage}{0.45\textwidth}
		\caption{DQN}
		\label{tab:DQN}
		\centering
		\begin{tabular}{cc}
			\toprule
			\textbf{Hyper-parameters} & \textbf{Value} \\
			\midrule
			Learning rate & 0.0005 \\
			Learning rate decay & True \\
			Total steps per episode & 5,000 \\
			Total training steps & 500K \\
			Optimizer & Adam W \\
			Mini batch size & 64 \\
			Discount factor $\gamma$ & 0.96 \\
			Buffer size & 1e6 \\
			Explore rate of $\epsilon$-greedy & 0.25 \\
			\bottomrule
		\end{tabular}
	\end{minipage}
\end{table}

\begin{table}[ht]
	\centering
	\begin{minipage}{0.45\textwidth}
		\caption{TS2C}
		\label{tab:TS2C}
		\centering
		\begin{tabular}{cc}
			\toprule
			\textbf{Hyper-parameters} & \textbf{Value} \\
			\midrule
			Learning rate & 0.0005 \\
			Learning rate decay & True \\
			Total steps per episode & 5,000 \\
			Total training steps & 500K \\
			Optimizer & Adam W \\
			Mini batch size & 64 \\
			Discount factor $\gamma$ & 0.96 \\
			Lambda entropy $\beta$ & 0.01 \\
			Clip parameter $\epsilon$ & 0.2 \\
			Lambda advantage $\lambda$ & 0.98 \\
			Intervention Threshold $\varepsilon'$ & 0.5 \\
			Intervention Minimization Ratio & 1 \\
			\bottomrule
		\end{tabular}
	\end{minipage}
	\hspace{0.05\textwidth} 
	\begin{minipage}{0.45\textwidth}
		\caption{SOAR-ACPPO}
		\label{tab:SOAR-ACPPO}
		\centering
		\begin{tabular}{cc}
			\toprule
			\textbf{Hyper-parameters} & \textbf{Value} \\
			\midrule
			Hyperparameter $q_{1}$ & 5 \\
			Hyperparameter $q_{2}$ & 10 \\
			Learning rate & 0.0005 \\
			Learning rate decay & True \\
			Total steps per episode & 5,000 \\
			Total training steps & 500K \\
			Optimizer & AdamW \\
			Mini batch size & 64 \\
			Discount factor $\gamma$ & 0.96 \\
			Lambda entropy $\beta$ & 0.01 \\
			Clip parameter $\epsilon$ & 0.2 \\
			Lambda advantage $\lambda$ & 0.98 \\
			Hyperparameter of AC $\psi$ & 0.2 \\
			\bottomrule
		\end{tabular}
	\end{minipage}
\end{table}

\end{document}